\documentclass[10pt,twocolumn,letterpaper]{article}

\usepackage{iccv}
\usepackage{times}
\usepackage{epsfig}
\usepackage{amsmath}
\usepackage{amssymb}
\usepackage{amsthm}
\usepackage{graphicx}
\usepackage[ruled]{algorithm2e}
\usepackage{algpseudocode}
\usepackage{microtype}
\usepackage{booktabs}
\usepackage{comment}
\usepackage{subcaption}
\usepackage{xcolor}
\usepackage{tabularx}
\usepackage{multirow}
\usepackage[accsupp]{axessibility}  

\usepackage[pagebackref=true,breaklinks=true,letterpaper=true,colorlinks,bookmarks=false]{hyperref}
\definecolor{ruby}{rgb}{0.88, 0.07, 0.37}
\definecolor{tealblue}{rgb}{0.21, 0.46, 0.53}

\newtheorem{lemma}{Lemma}

\newcolumntype{Y}{>{\centering\arraybackslash}X}
\iccvfinalcopy 

\def\iccvPaperID{8592} 

\ificcvfinal\fi

\begin{document}
\title{Point-set Distances for Learning Representations of 3D Point Clouds}
\author{
Trung Nguyen$^{1}$ \quad
Quang-Hieu Pham$^{3}$ \quad
Tam Le$^{4}$ \quad 
Tung Pham$^{1}$ \quad
Nhat Ho$^{5}$ \quad 
Binh-Son Hua$^{1,2}$\\
$^{1}$VinAI Research, Vietnam \quad $^{2}$VinUniversity, Vietnam \\$^{3}$Woven Planet North America, Level 5 ~~
$^{4}$RIKEN AIP, Japan ~~ $^{5}$University of Texas, Austin \\
}
\maketitle
\ificcvfinal\thispagestyle{empty}\fi

\begin{figure*}
  \centering
  \newcommand{\sizea}{0.138\linewidth}
  \includegraphics[width=\sizea]{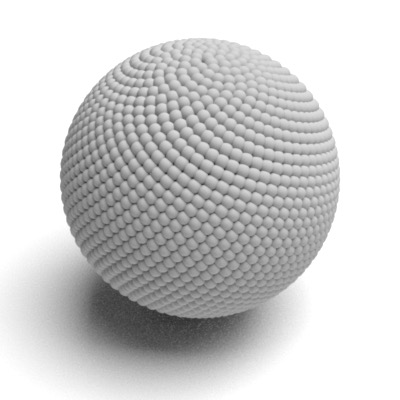}%
  \includegraphics[width=\sizea]{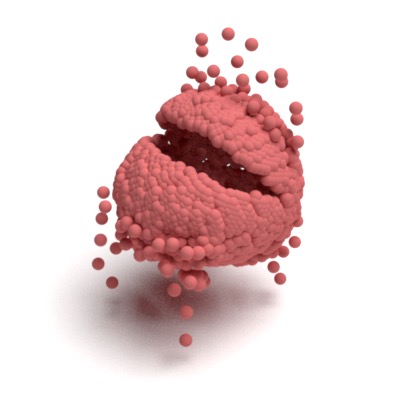}%
  \includegraphics[width=\sizea]{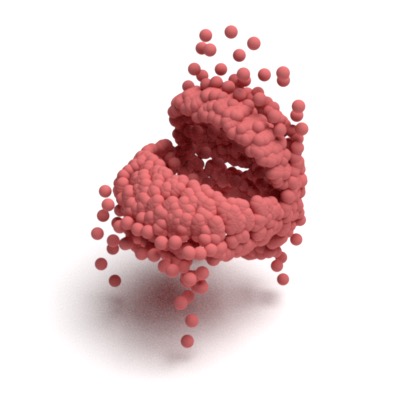}%
  \includegraphics[width=\sizea]{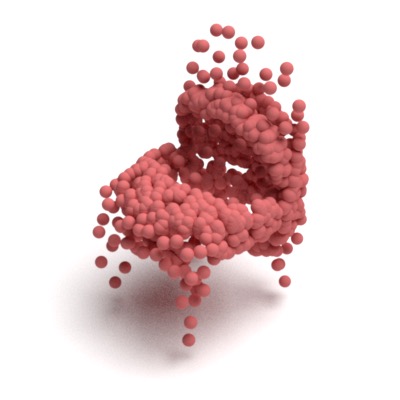}%
  \includegraphics[width=\sizea]{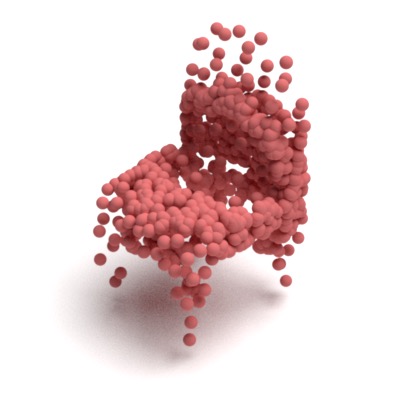}%
  \includegraphics[width=\sizea]{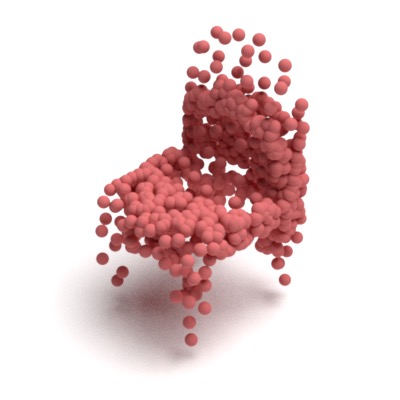}%
  \includegraphics[width=\sizea]{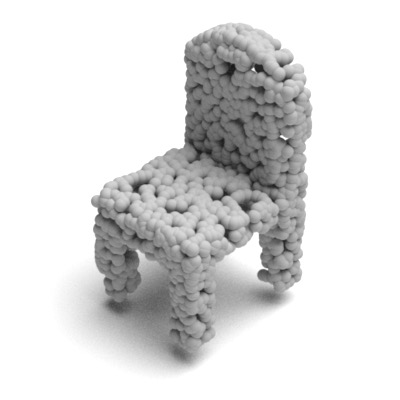}\\
  \includegraphics[width=\sizea]{figures/teaser/sphere.png}%
  \includegraphics[width=\sizea]{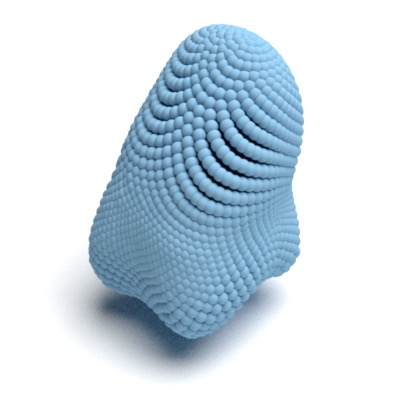}%
  \includegraphics[width=\sizea]{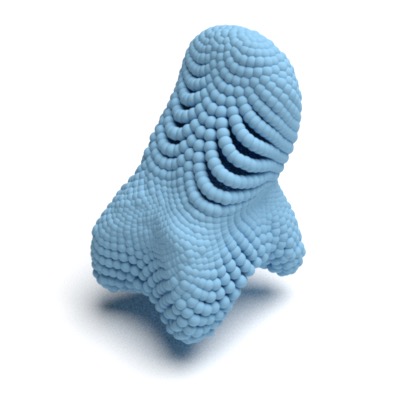}%
  \includegraphics[width=\sizea]{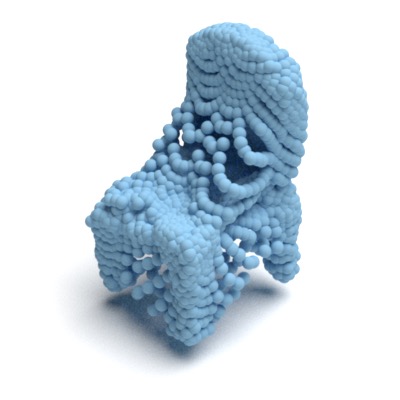}%
  \includegraphics[width=\sizea]{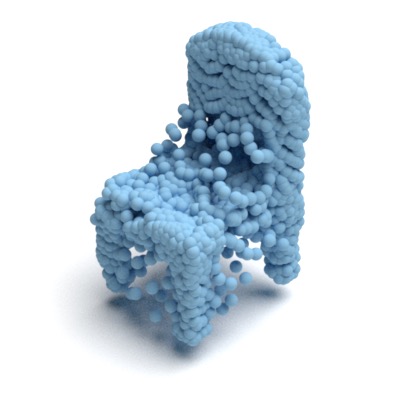}%
  \includegraphics[width=\sizea]{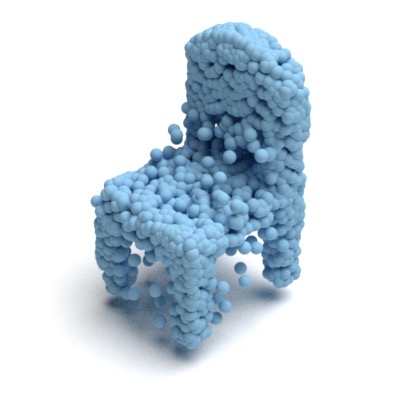}%
  \includegraphics[width=\sizea]{figures/teaser/chair.png}
  \caption{
    We advocate the use of sliced Wasserstein distance for training 3D point cloud autoencoders.
    In this example, we try to morph a sphere into a chair by optimizing two different loss functions:
    Chamfer discrepancy (top, red) and sliced Wasserstein distance (bottom, blue).
    The proposed sliced Wasserstein distance only takes 1000 iterations to converge, while it takes 50000 iterations for Chamfer discrepancy.
  }
  \label{fig:teaser}
\end{figure*}

\begin{abstract}
  Learning an effective representation of 3D point clouds requires a good metric to measure the discrepancy between two 3D point sets, which is non-trivial due to their irregularity.
  Most of the previous works resort to using the Chamfer discrepancy or Earth Mover's distance, but those metrics are either ineffective in measuring the differences between point clouds or computationally expensive.
  In this paper, we conduct a systematic study with extensive experiments on distance metrics for 3D point clouds.
  From this study, we propose to use sliced Wasserstein distance and its variants for learning representations of 3D point clouds.
  In addition, we introduce a new algorithm to estimate sliced Wasserstein distance that guarantees that the estimated value is close enough to the true one.
  Experiments show that the sliced Wasserstein distance and its variants allow the neural network to learn a more efficient representation compared to the Chamfer discrepancy.
  We demonstrate the efficiency of the sliced Wasserstein metric and its variants on several tasks in 3D computer vision including training a point cloud autoencoder, generative modeling, transfer learning, and point cloud registration.
\end{abstract}

\section{Introduction}

Since the spark of the modern artificial intelligence, 3D deep learning on point clouds has become a powerful technique for solving recognition problems such as object classification~\cite{qi2017pointnet,hua2018pointwise}, object detection~\cite{qi2019votenet}, and semantic segmentation~\cite{pham2019jsis3d}.
Generative modeling with 3D point clouds has also been studied with some promising results~\cite{yang2019pointflow,shu2019tree,li2018pcgan,hui2020pdgn,sun2020pointgrow,lin2018learning}.
Another 3D computer vision problem that has seen the rise of deep learning approaches is point cloud matching~\cite{choy2019fcgf,deng2018ppfnet,deng2018ppffoldnet,gojcic2019perfectmatch}.
All of these problems share a common task --- that is to learn a robust representation of 3D point clouds.

One of the most important steps in learning representations of 3D point clouds is to choose a metric to measure the discrepancy between two point sets.
There are two 
popular choices for such metric: the Chamfer 
divergence and the Earth Mover's distance (EMD)~\cite{fan2017pointset}.
While earlier works~\cite{fan2017pointset,achlioptas2018pointnetae} has shown that EMD performs better than Chamfer in terms of learning representations,
Chamfer is more 
favored~\cite{deng2018ppffoldnet,yang2018foldingnet,groueix2018atlasnet,duan2019denoising,deng2019locfeat,hermosilla2019total} due to its significantly lower computational cost.

In this article, we revisit the similarity metric problem in 3D point cloud deep learning. We propose to use the sliced Wasserstein distance (\textsc{SWD})~\cite{bonneel2015sliced}, which is based on projecting the points in point clouds 
into a line, and its variants as effective metrics to supervise 3D point cloud autoencoders. Compared to Chamfer 
divergence, SWD is more suitable for point cloud reconstruction, while remaining computationally efficient (cf. Figure~\ref{fig:teaser}). 
We show that Chamfer divergence is weaker than the EMD and sliced Wasserstein distance (cf. Lemma~\ref{lemma:relation_Chamfer_Wasserstein}) while the EMD and sliced Wasserstein distance are equivalent. It suggests that even when two point clouds are close in Chamfer divergence, they may not be close in either the EMD or sliced Wasserstein distance. Furthermore, the sliced Wasserstein distance has a computational complexity in the order of $N \log N$~\cite{bonneel2015sliced}, which is comparable to that of the Chamfer divergence, while EMD has a complexity in the order of $N^3$~\cite{Pele-2009-Fast} where $N$ is the number of points in 3D point clouds. Finally, under the standard point clouds settings, since the dimension of points is usually three, the projection step in sliced Wasserstein distance will only lead to small loss of information of the original point clouds. As a consequence, the sliced Wasserstein distance possesses both the computational and statistical advantages for point cloud learning over Chamfer divergence and EMD. To improve the quality of slices from SWD, we also discuss variants of sliced Wasserstein distance, including max-sliced Wasserstein distance~\cite{deshpande2019max} and the proposed adaptive sliced Wasserstein algorithm.
By conducting a case study on learning a 3D point cloud auto-encoder, we provide a comprehensive benchmark on the performance of different metrics.
These results align with our theoretical development.
In summary, our main findings are:

\begin{itemize}
\item A first theoretical study about the relation between Chamfer divergence, EMD, and sliced Wasserstein distance for point cloud learning (Section \ref{sec:main_cont}).
\item A new algorithm named Adaptive-sliced Wasserstein to evaluate sliced Wasserstein distance that guarantees that the evaluated value is close enough to the true one (Section \ref{sec:main_cont}).
\item An extensive evaluation of point cloud learning tasks including point cloud reconstruction, transfer learning, point cloud registration and generation based on Chamfer discrepancy, EMD, sliced Wasserstein distance and its variants (Section \ref{sec:exp}).
\end{itemize}

\section{Related Work}

\paragraph{Set similarity.}
3D point clouds autoencoders are useful in a wide range of applications, such as denoising~\cite{hermosilla2019total}, 3D matching~\cite{zhao2019pointcapsule,deng2018ppffoldnet,deng2019locfeat, Li_2019_CVPR}, and generative models~\cite{fan2017pointset,achlioptas2018pointnetae}.
Many autoencoder architectures have been proposed in recent few years~\cite{achlioptas2018pointnetae,yang2018foldingnet,deng2018ppffoldnet}.
To train these autoencoders, there are two popular choices of losses: the Chamfer discrepancy (CD) and Earth Mover's distance (EMD).
The Chamfer discrepancy has been widely used in point cloud deep learning~\cite{fan2017pointset,achlioptas2018pointnetae,yang2018foldingnet}.

It is known that Chamfer discrepancy (CD) is not a distance that means there are two different point clouds with its CD almost equals zero.
While earlier works~\cite{fan2017pointset,achlioptas2018pointnetae} showed that EMD is better than Chamfer in 3D point clouds reconstruction task,
recent works~\cite{pham2020lcd} still favor Chamfer 
discrepancy due to its fast computation. 

\vspace{0.5 em}
\noindent
\textbf{Wasserstein distance.} In 2D computer vision, the family of Wasserstein distances and their sliced-based versions have been considered in the previous works~\cite{Arjovsky-2017-Wasserstein, ho2017multilevel, tolstikhin2018wasserstein, deshpande2018generative,deshpande2019max,wu2019sliced,ho2021lamda,Ho_JMLR}. In particular, Arjovsky et al.~\cite{Arjovsky-2017-Wasserstein} proposed using Wasserstein as the loss function in generative adversarial networks (GANs) while Tolstikhin et al.~\cite{tolstikhin2018wasserstein} proposed using that distance for the autoencoder framework. Nevertheless, Wasserstein distances, including the EMD, have expensive computational cost and can suffer from the curse of dimensionality, namely, the number of data required to train the model will grow exponentially with the dimension. To deal with these issues of Wasserstein distances, a line of works has utilized the slicing approach to reduce the dimension of the target probability measures. The notable slicing distance is sliced Wasserstein distance~\cite{bonneel2015sliced}. Later, the idea of sliced Wasserstein distance had been adapted to the autoencoder setting~\cite{kolouri2018_sliced} and domain adaptation~\cite{Lee_2019_CVPR}. 
Deshpande et al.~\cite{deshpande2018generative} proposed to use the max-sliced Wasserstein distance, a version of sliced Wasserstein distance when we only choose the best direction to project the probability measures, to formulate the training loss for a generative adversarial network. 
The follow-up work~\cite{deshpande2019max,wu2019sliced} has an improved projection complexity compared to the sliced Wasserstein distance. In the recent work, Nguyen et al.~\cite{Nguyen_distributional, Nguyen_vmf} proposed a probabilistic approach to chooses a number of important directions via finding an optimal measure over the projections. 
Another direction with sliced-based distances is by replacing the linear projections with non-linear projections to capture more complex geometric structures of the probability measures~\cite{gsw_nips19}. 
However, to the best of our knowledge, none of such works have considered the problem of learning with 3D point clouds yet.\\
\noindent
\textbf{Notation.}
Let $\mathbb{S}^{n-1}$ be the unit sphere in the $n$-dimensional space. For a metric space $(\Omega,d)$ and two probability measures $\mu$ and $\nu$ on $\Omega$, let $\Pi(\mu,\nu)$ be the set of all joint distributions $\gamma$ such that its marginal distributions are $\mu$ and $\nu$, respectively. For any $\theta \in \mathbb{S}^{n - 1}$ and any measure $\mu$, $\pi_{\theta} \sharp \mu$ denotes the pushforward measure of $\mu$ through the mapping $\mathcal{R}_{\theta}$ where $\mathcal{R}_{\theta}(x) = \theta^{\top} x$ for all $x$.


\section{Background}

To study the performance of different metrics for point cloud learning, 
we briefly review the mathematical foundation of the Chamfer discrepancy  and the Wasserstein distance, which serve as the key building blocks in this paper. Note that in computer vision, Chamfer is often abused to be a distance. Strictly speaking, Chamfer is a pseudo-distance, \emph{not} a distance~\cite{fan2017pointset}. Therefore,
in this paper, we use the terms Chamfer discrepancy or Chamfer divergence instead. 

\subsection{Chamfer discrepancy}
In point cloud deep learning, Chamfer discrepancy has been adopted for many tasks. There are some variants of Chamfer discrepancy, which we provide here for completeness.
For any two point clouds $P, Q$, a common formulation of the Chamfer discrepancy between $P$ and $Q$ is given by:
\begin{equation}
    \label{chamfer_cvpr17}
d_{\text{CD}}(P, Q) = \frac{1}{|P|}\sum_{x\in P}\min_{y\in Q}||x - y||^2_2 + \frac{1}{|Q|}\sum_{y\in Q}\min_{x\in P}||x - y||^2_2.
\end{equation}
A slightly modified version of Chamfer divergence is also used by previous works~\cite{yang2018foldingnet,deng2018ppffoldnet,deng2019locfeat,aumentado2019geometric} that replaces the sum by a $\max$ function:
\begin{align}
    \label{modified_chamfer}
    d_{\text{MCD}}(P, Q) = \max\Big\{&\frac{1}{|P|}\sum_{x\in P}\min_{y\in Q}||x - y||^2_2, \nonumber \\ 
    &\frac{1}{|Q|}\sum_{y\in Q}\min_{x\in P}||x - y||^2_2\Big\}.
\end{align}
In both definitions, the $\min$ function means that Chamfer discrepancy only cares about the nearest neighbour of a point rather than the distribution of those nearest points.
Hence, as long as the supports of $x$ and $y$ are close then the corresponding Chamfer discrepancy between them is small, meanwhile their corresponding distributions could be different.
A similar phenomenon, named Chamfer blindness, was shown in~\cite{achlioptas2018pointnetae}, pointing out that Chamfer discrepancy fails to distinguish bad sample from the true one, since it is less discriminative.

\subsection{Wasserstein distance}
Let $(\Omega, d)$ be a metric space, $\mu, \nu$ are probability measures on $\Omega$. 
For $p\geq1$, the $p$-Wasserstein distance (WD) is given by
\begin{equation*}
    W_p(\mu, \nu)=\underset{\gamma\in\Pi(\mu, \nu)}{\inf}\Big\{\mathbb{E}_{(X, Y)\sim\gamma}\big[d^p(X, Y)\big]\Big\}^{\frac{1}{p}}.
\end{equation*}
For $p=1$, the Wasserstein distance becomes the Earth Mover's Distance (EMD), where the optimal joint distribution $\gamma$ induces a map  $T:\mu\rightarrow \nu$, which preserves the measure on any measurable set $B \subset \Omega$.
In the one-dimensional  case $\Omega=\mathbb{R}$ and $d(x,y):=|x-y|$, the WD has the following  closed-form formula:
\begin{equation}
    \label{equation:wd1d}
    W_p(\mu, \nu) = \Big(\int_0^1\big|F^{-1}_X(t) - F^{-1}_Y(t)\big|^p, dt\Big)^{\frac{1}{p}},
\end{equation}
where $F_{X}$ and $F_{Y}$ are respectively the cumulative distribution functions of random variables $X$ and $Y$. When the dimension is greater than one, there is no closed-form for the WD, that makes calculating the WD more difficult.  

\vspace{0.5 em}
\noindent
\textbf{EMD in 3D point-cloud applications.} For the specific settings of 3D point clouds, the EMD had also been employed to define a metric between two point clouds~\cite{fan2017pointset,achlioptas2018pointnetae,yu2018punet}. With an abuse of notation, for any two given point clouds $P$ and $Q$, throughout this paper, we denote its measure representation as follows: $P = \frac{1}{|P|} \sum_{x \in P} \delta_{x}$ and $Q = \frac{1}{|Q|} \sum_{y \in Q} \delta_{y}$ where $\delta_{x}$ denotes the Dirac delta distribution at point $x$ in the point cloud $P$. 
When $|P|=|Q|$, the Earth Mover's distance~\cite{fan2017pointset,achlioptas2018pointnetae,yu2018punet} between $P$ and $Q$ is defined as
\begin{equation}
    \label{emd_cvpr17}
    d_{\text{EMD}} (P, Q) = \underset{T:P\to Q}{\min}\sum_{x\in P}||x - T(x)||_2.
\end{equation}
While earlier works~\cite{fan2017pointset,achlioptas2018pointnetae} showed that EMD is better than Chamfer in 3D point clouds reconstruction task, the computation of EMD can be very expensive compared to the Chamfer divergence. In particular, it had been shown that the practical computational efficiency of EMD is at the order of $\max\{|P|, |Q|\}^3$~\cite{Pele-2009-Fast}, which can be expensive. There is a recent line of work using entropic version of EMD or in general Wasserstein distances~\cite{cuturi2013sinkhorn} to speed up the computation of EMD. However, the best known practical complexity of using the entropic approach for approximating the EMD is at the order $\max\{|P|, |Q|\}^2$~\cite{lin2019efficient, Linacceleration}, which is still expensive and slower than that of Chamfer divergence. Therefore, it necessitates to develop a metric between 3D point clouds such that it not only has equivalent statistical properties as those of EMD but also has favorable computational complexity similar to that of the Chamfer divergence.











\section{Sliced Wasserstein Distance and its variants}
\label{sec:main_cont}
In this section, we first show that Chamfer divergence is a weaker divergence than Earth Mover's distance in Section~\ref{sec:curse_dimensionality}. Since Earth Mover's distance can be expensive to compute, we propose using sliced Wasserstein distance, which is equivalent to Wasserstein distance and has efficient computation, as an alternative to EMD in Section~\ref{subsec:sliced_Wasserstein}. Finally, we propose a new algorithm to compute sliced-Wasserstein that can guarantee the estimated value is close enough to the true one in Section~\ref{subsec:asw}.

\subsection{Relation between Chamfer divergence and Earth Mover's distance}
\label{sec:curse_dimensionality}
In this section, we study the relation between Chamfer and EMD when $|P| = |Q|$. In particular,  the following inequality shows that the Chamfer divergence is weaker than the Wasserstein distance.
\begin{lemma}
\label{lemma:relation_Chamfer_Wasserstein}
Assume  $|P| = |Q|$ and the support of $P$ and $Q$ is bounded in a convex hull of diameter $K$, then  we find that
\begin{align}
    d_{CD}(P, Q) \leq 2 K d_{EMD}(P, Q).
\end{align}
\end{lemma}
\begin{proof}
Assume that $T$ is the optimal plan from $P$ to $Q$. Then, we find that
\begin{align*}
  \min_{y\in Q} \|x-y\|_2&\leq \|x - T(x)\|_2\\
\Rightarrow\min_{y\in Q} \|x-y\|^2_2 &\leq  K \|x - T(x)\|_2,
\end{align*}
since $\|x - T(x)\|_2 \leq K$.
Taking the sum over $x$, we obtain
\begin{align*}
     \sum_{x\in P} \min_{y\in Q} \|x-y\|_2^2 \leq K \sum_{x\in P} \|x-T(x)\|_2.
\end{align*}
Similarly we have  $\sum_{y\in Q} \min_{x\in P} \|x-y\|_2^2 \leq K \sum_{y\in Q} \|y-\bar{T}(y)\|_2$ where $\bar{T}$ is the optimal plan from $Q$ to $P$. Then we obtain the desired inequality.
\end{proof}
The inequality in Lemma~\ref{lemma:relation_Chamfer_Wasserstein} implies that minimizing the Wasserstein distance leads to a smaller Chamfer discrepancy, and the reverse inequality is not true. Therefore, EMD and Chamfer divergence are not equivalent, which can be undesirable. Note that the inequality  could not be improved to a better order of $K$. For example, lets consider two point-clouds that have small variance  $\epsilon^2$, meanwhile the distance between two point-cloud centers equals $K - O(\epsilon)$. Then the CD is of order $K^2$, the EMD is of order $K$.

Lemma 1 shows that Chamfer divergence is weaker than the EMD, which is in turn weaker than other divergences such as, KL, chi-squared, etc.~\cite[pp.117]{peyre2020computational}. However, Chamfer discrepancy has weakness as we explained before, and other divergences are not as effective as EMD, since they are very loose even when two distributions are close to each other~\cite{Arjovsky-2017-Wasserstein}. Hence, the Wasserstein/EMD is a preferable metric for learning the discrepancy between two distributions.

Despite that fact, computing EMD can be quite expensive since it is equivalent to solving a linear programming problem, which has the best practical computational complexity of the order $\mathcal{O}(\max \{|P|^3, |Q|^3\})$~\cite{Pele-2009-Fast}. On the other hand, the computational complexity of Chamfer divergence can be scaled up to the order $\mathcal{O}(\max \{|P|, |Q|\})$. Therefore, Chamfer divergence is still preferred in practice due to its favorable computational complexity. 

Given the above observation, ideally, we would like to utilize a distance between $P$ and $Q$ such that it is equivalent to the EMD and has linear computational complexity in terms of $\max \{|P|, |Q|\}$, which is comparable to the Chamfer divergence. It leads us to the notion of sliced Wasserstein distance in the next section.

\subsection{Sliced Wasserstein distance and its variants}
\label{sec:sliced_Wasserstein}
In order to circumvent the high computational complexity of EMD, the sliced Wasserstein distance~\cite{bonneel2015sliced} is designed to exploit the 1D formulation of Wasserstein distance in Equation~\eqref{equation:wd1d}. 
\subsubsection{Sliced Wasserstein distance} 
\label{subsec:sliced_Wasserstein}
In particular, the idea of sliced Wasserstein distance is that we first project both target probability measures $\mu$ and $\nu$ on a direction, says $\theta$, on the unit sphere to obtain two projected measures denoted by $\pi_{\theta} \sharp \mu$ and $\pi_{\theta} \sharp \nu$, respectively. Then, we compute the Wasserstein distance between two projected measures $\pi_{\theta} \sharp \mu$ and  $\pi_{\theta} \sharp \nu$. The sliced Wasserstein distance (SWD) is defined by taking the average of the Wasserstein distance between the two projected measures over all possible projected direction $\theta$. In particular, for any given $p \geq 1$, the sliced Wasserstein distance of order $p$ is formulated as follows:
\begin{equation}
    SW_p(\mu, \nu) = \Big( \int_{\mathbb{S}^{n-1}} W_p^p\big(\pi_\theta \sharp\mu, \pi_\theta \sharp \nu)d\theta \Big)^{\frac{1}{p}}. \label{eq:sliced_Wasserstein}
\end{equation}
The $SW_p$ is considered as a low-cost approximation for Wasserstein distance as its computational complexity is of the order 
$\mathcal{O}(n \log n)$ where $n$ is the maximum number of supports of the discrete probability measures $\mu$ and $\nu$. When $p = 1$, the $SW_{p}$ is weakly equivalent to first order WD or equivalently EMD~\cite{bayraktar2019strong}. 
The equivalence between $SW_{1}$ and EMD along with the result of Lemma~\ref{lemma:relation_Chamfer_Wasserstein} suggests that $SW_{1}$ is stronger metric than the Chamfer divergence while it has an appealing optimal computational complexity that is linear on the number of points of point clouds, which is comparable to that of Chamfer divergence. 

We  would like to remark that since the dimension of points in point clouds is generally small ($\leq 6$), sliced Wasserstein distance will still be able to retain useful information of the point clouds even after we project them to some direction on the sphere. Due to its favorable computational complexity, sliced Wasserstein distance had been used in several applications: point cloud registration \cite{lai2014multiscale}, generative models on 2D images; see, for examples \cite{ort_est_wd_aistats19,asy_gua_gm_sw_nips19,deshpande2018generative,kolouri2018_sliced,sw_gau_mix_cvpr18,wu2019sliced,sw_ker_pro_dis_cvpr16}. However, to the best of our knowledge, this distance has not been used for deep learning tasks on 3D point clouds.

\vspace{0.5 em}
\noindent
\textbf{Monte Carlo estimation.}
In Equation~\eqref{eq:sliced_Wasserstein}, the integral is generally intractable to compute. Therefore, we need to approximate the integral. A common approach to approximate the integral is by applying the Monte Carlo method. In particular, we sample $N$ directions $\theta_{1}, \ldots, \theta_{N}$ uniformly from the sphere 
$\mathbb{S}^{\texttt{d}-1}$ where $\texttt{d}$ is the dimension of points in point clouds, which results in the following approximation of sliced Wasserstein distance:
\begin{align}
    SW_p(\mu,\nu) \approx \Big(\frac{1}{N} \sum_{i=1}^N W_p^p\big(\pi_{\theta_i} \sharp\mu, \pi_{\theta_i} \sharp \nu\big) \Big)^{\frac{1}{p}}. \label{eq:approximation_SW}
\end{align}
where the number of slices $N$ is tuned for the best performance. 

Since $N$ plays a key role in determining the approximation of sliced Wasserstein distance, it is usually chosen based on the dimension of the probability measures $\mu$ and $\nu$. In our applications with 3D point clouds, since the dimension of the points in point clouds is generally small, we observe that choosing the number of projections $N$ up to 100 is already sufficient for learning 3D point clouds well.

\vspace{0.5 em}
\noindent
\textbf{Max-sliced Wasserstein distance.} To avoid using uninformative slices in SWD, another approach focuses on taking only the best slice in discriminating between two given distributions. That results in max-sliced Wasserstein distance~\cite{deshpande2019max}. 
For any $p \geq 1$, the \emph{max-sliced Wasserstein distance} of order $p$ is given by:
\begin{equation}
    MSW_p(\mu, \nu) : = \max_{\theta \in \mathbb{S}^{n-1}} W_p \big(\pi_\theta \sharp\mu, \pi_\theta \sharp \nu). \label{eq:max_sliced_Wasserstein}
\end{equation}

\subsubsection{An adaptive-sliced Wasserstein algorithm}
\label{subsec:asw}
Another drawback of the Monte Carlo estimation in SWD is that it does not give information about how close the estimated value is to the true one. Therefore, we introduce the novel \emph{adaptive-sliced Wasserstein algorithm} (ASW). 
In particular, given $N$ uniform random projections $\{\theta_i \}_1^{N}$ drawn from the sphere $\mathbb{S}^{n-1}$, for the simplicity of presentation, we denote $sw_{i} = W_p^p\big(\pi_{\theta_{i}} \sharp\mu, \pi_{\theta_{i}} \sharp \nu)$ for all $1 \leq i \leq N$. Furthermore, we denote $\overline{sw} = SW_p^{p}(\mu,\nu)$ the true value of SW distance that we want to compute. The Monte Carlo estimation of the SW distance can be written as follows: $\overline{sw}_N := \frac{1}{N}\sum_{i=1}^{N}sw_i$,
which is an unbiased estimator of the true value $\overline{sw}$. Similarly, the biased and unbiased variance estimates are respectively defined as:
\begin{align}
s^2_N := \frac{1}{N} \sum_{i=1}^{N}\big(sw_i - \overline{sw}_N \big)^2 , \quad \bar{s}^2_N := \frac{N}{N-1} s^2_N.
\label{eqn:biased_var}
\end{align}
Our idea of adaptivity is to dynamically determine the number of projections $N$ from the observed mean and variance of the estimators. To do so, we leverage a probabilistic bound of the error of the estimator and choose $N$ such that the error bound is below a certain tolerance threshold. Particularly, applying central limit theorem, we have
\begin{align}
\label{k-sigma}
    \mathbb{P}\Big(|\overline{sw}_N - \overline{sw}|<\frac{k\bar{s}_N}{\sqrt{N}}\Big)\approx \phi(k) - \phi(-k)
\end{align}
where $\phi(k) := \int_{-\infty}^{k}\frac{e^{-x^2/2}}{\sqrt{2\pi}}dx$. We set $k = 2$ so that the above probability is around 95\%.

\begin{algorithm}[t]
\textbf{Input}: Two point sets, positive integers $N_0, s$; $\epsilon > 0$; maximum number of projections $M$\\
\textbf{Ouput}: $\overline{sw}_N$ \\
 Sample $N_0$ projections \;
 Compute $\overline{sw}:=\overline{sw}_{N_0}, \overline{sw^2}:=\overline{sw^2}_{N_0}, N:=N_0$ \;
 \While{$\overline{sw^2} - (\overline{sw})^2 > \frac{(N - 1)\epsilon^2}{4}$ \& $N \leq M$}{
  Sample $s$ projections \;
  Compute $\overline{sw}_{s}, \overline{sw^2}_{s}$ \;
  Assign $\overline{sw} := \frac{N \times \overline{sw} + s\times\overline{sw}_s}{N + s}$ \;
  Assign $\overline{sw^2} := \frac{N \times \overline{sw^2} + s\times\overline{sw^2}_s}{N + s}$ \;
  Assign $N := N + s$ \;
 }
\caption{Adaptive sliced Wasserstein.}
\label{alg:adaptive}
\end{algorithm}
\noindent
Given a predefined tolerance $\epsilon>0$, we aim for
\begin{align}
\frac{k\bar{s}_N}{\sqrt{N}}\leq\epsilon, \ \text{or} \ \frac{k^2\bar{s}^2_N}{N} \leq \epsilon^2. \label{stop_condition}
\end{align}
From Equation~(\ref{eqn:biased_var}), we note that $\frac{\bar{s}^2_N}{N} = \frac{s^2_N}{N - 1}$
and so it is desirable to choose $N$ such that $\frac{k^2 s^2_N}{N - 1} \leq \epsilon^2$. Rewriting the biased variance in Equation~(\ref{eqn:biased_var}), we get $s^2_N = \frac{1}{N}\sum_{i=1}^N sw_i^2 - (\overline{sw}_N)^2$. Denote $\overline{sw^2}_N := \frac{1}{N}\sum_{i=1}^N sw_i^2$, the condition becomes
\begin{align*}
    \overline{sw^2}_N - (\overline{sw}_N)^2 \leq \frac{(N-1)\epsilon^2}{k^2}.
\end{align*}
That leads us to the construction of Algorithm~\ref{alg:adaptive}. In this algorithm, we start by estimating the SWD with an initial number of projections $N = N_0$, and then dynamically update $N$ with extra $s$ samples by estimating the online mean and variance of the distance estimator until the estimated error satisfies the error bound.
We note that ASW algorithm can be used to compute other variants of SWD, such as, generalized sliced-Wasserstein distance~\cite{gsw_nips19}.
\section{Experiments}
\label{sec:exp}

\begin{figure}[t]
  \centering
  \includegraphics[width=\linewidth]{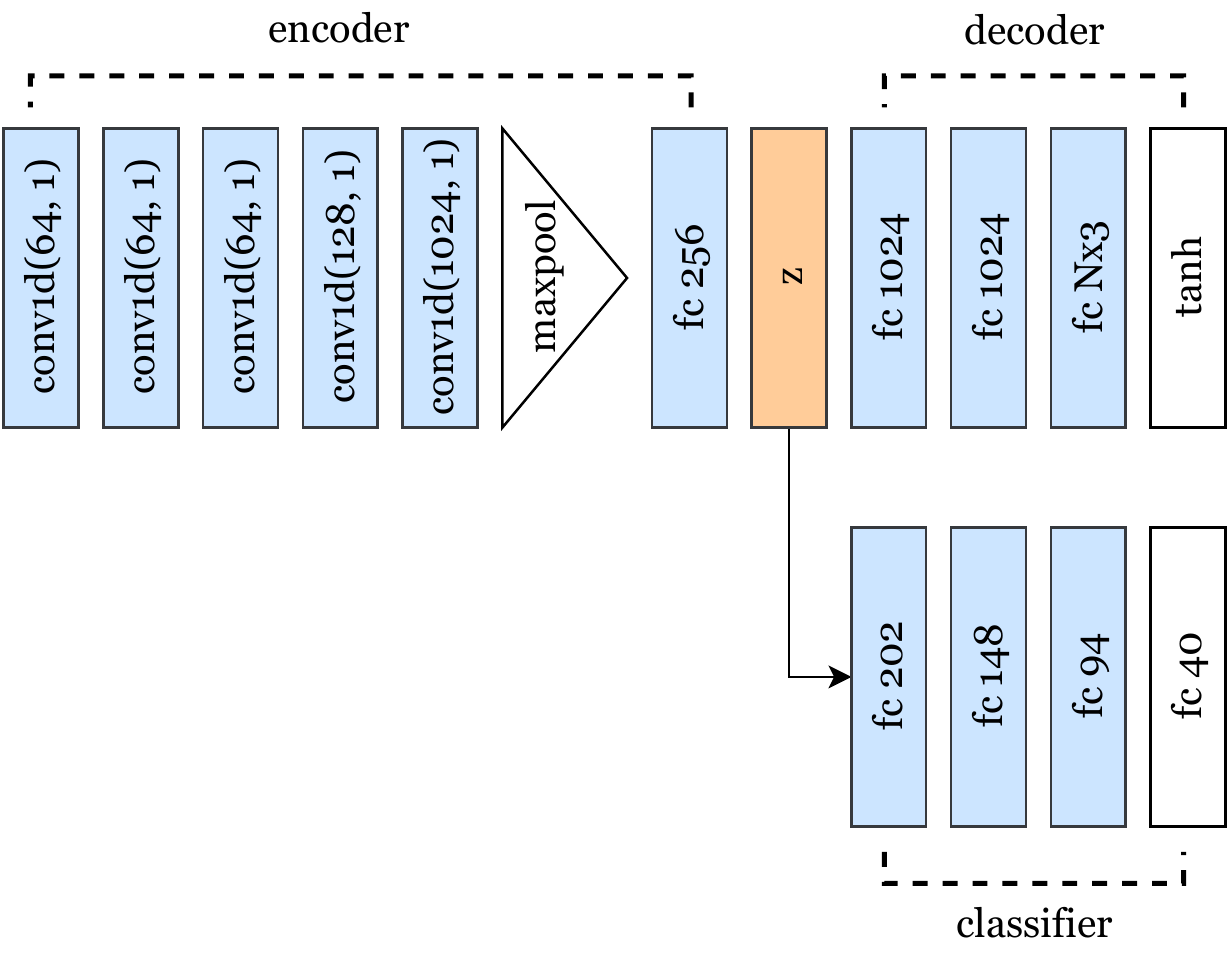}
  \caption{
    The network architecture of the autoencoder used in all of our experiments.
    The classifier is only used in the transfer learning experiment.
    All layers are followed by ReLU activation and batch normalization by default, except for the final layers.
  }
  \label{fig:network}
\end{figure}

In general, a good distance metric is expected to have good performance in a wide range of downstream tasks.
Here we compare the performance of different autoencoders trained with Chamfer discrepancy, Earth Mover's distance, sliced Wasserstein distance and max-sliced Wasserstein distance. 
We consider the following tasks in our evaluation: point cloud reconstruction, transfer learning, point cloud registration, and point cloud generation.

\vspace{0.5 em}
\noindent
\textbf{Implementation details.}
We follow the same architecture of the autoencoder used in~\cite{pham2020lcd}, which is based on PointNet~\cite{qi2017pointnet}, with 256-dimensional embedding space.
The architecture of our autoencoder is shown in Figure~\ref{fig:network}.
We train the autoencoder on the ShapeNet Core-55 dataset~\cite{chang2015shapenet} for 300 epochs with the following loss functions: Chamfer discrepancy (\textsc{CD-AE}), Earth-mover distance (\textsc{EMD-AE}), max-sliced Wasserstein distance (\textsc{MSW-AE}) and sliced Wasserstein distance. In the case which the autoencoder trained with sliced Wasserstein distance, we conduct experiments for each of two algorithms: Monte Carlo estimation and ASW. We will call these two autoencoders \textsc{SSW-AE} and \textsc{ASW-AE}, respectively. For Monte Carlo estimation, we set the number of slices 100. For ASW, we set the parameters in Algorithm~\ref{alg:adaptive} as follows: $N_0=2$, $s=1$, $\epsilon=0.5$ and $M=500$. Our models are trained with an SGD optimizer with an initial learning rate $0.001$, a momentum of $0.9$, and a weight decay of $0.0005$.
We use an NVIDIA V100 GPU for both training and evaluation, with batch size of 128 and a point cloud size of 2048.
\begin{table}[t]
  \centering
  \begin{tabular}{lccc}
    \toprule
    Method                 & CD             & SWD            & EMD            \\
    \midrule
    \textsc{CD-AE}         & 0.014          & 6.738          & 0.314          \\
    \textsc{EMD-AE}        & 0.014          & 2.295          & 0.114          \\
    \textsc{SSW-AE} (ours) & \textbf{0.007} & \textbf{0.831} & \textbf{0.091} \\
    \textsc{MSW-AE} (ours) & \textbf{0.007} & 0.865          & 0.093 \\
    \textsc{ASW-AE} (ours) & \textbf{0.007} & 0.854          & 0.092 \\
    \bottomrule
  \end{tabular}
  \caption{
    Quantitative measurements of the discrepancy between the input point clouds and their reconstructed versions on ModelNet40.
    We use Chamfer discrepancy (CD), sliced Wasserstein distance with 100 slices (SWD), and EMD as the evaluation metrics.
  }
  \label{tab:recon}
\end{table}

\begin{table}[t]
  \centering
  \begin{tabular}{lc}
    \toprule
    Method                 & Accuracy (\%) \\
    \midrule
    \textsc{CD-AE}         & 83.9          \\
    \textsc{EMD-AE}        & 84.4          \\
    \textsc{SSW-AE} (ours) & \textbf{86.8} \\
    \textsc{MSW-AE} (ours) & 86.5\\
    \textsc{ASW-AE} (ours) & \textbf{86.8} \\
    \bottomrule
  \end{tabular}
  \caption{
    Classification performance of different autoencoders on ModelNet40~\cite{wu2015modelnet}.
    Our proposed SW models can learn a better latent representation compared to Chamfer and EMD.
  }
  \label{tab:classification}
\end{table}

\vspace{0.5 em}
\noindent
\textbf{3D point cloud reconstruction.}
We test the reconstruction capability of the autoencoders on the ModelNet40 dataset~\cite{wu2015modelnet}.
We measure the differences between the original point clouds and their reconstructed versions using Chamfer discrepancy (CD), sliced Wasserstein distance with 100 slices (SWD), and EMD.
The results are shown in Table~\ref{tab:recon}.
Figure~\ref{fig:recon} shows the qualitative results from different autoencoders.
As can be seen, \textsc{CD-AE} performs well when being evaluated by Chamfer discrepancy, but not by SWD and EMD.
On the contrary, from Lemma~\ref{lemma:relation_Chamfer_Wasserstein}, minimizing Wasserstein distance leads to minimizing Chamfer distance as well.
Experiments with other Wasserstein metrics including generalized sliced Wasserstein~\cite{gsw_nips19} could be found in the supplementary material.

\begin{figure*}
  \centering
  \newcommand{\sizea}{0.164\linewidth}
  \includegraphics[width=\sizea]{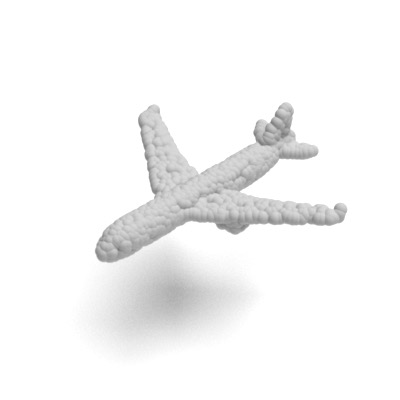}%
  \includegraphics[width=\sizea]{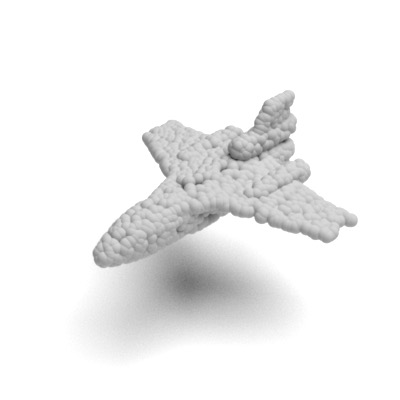}%
  \includegraphics[width=\sizea]{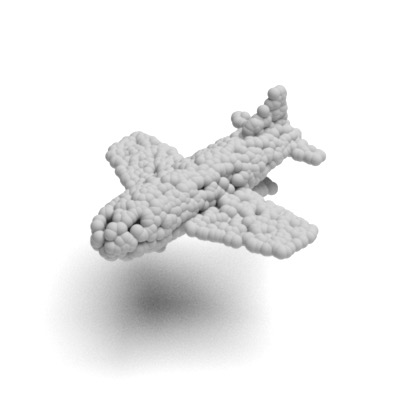}
  \includegraphics[width=\sizea]{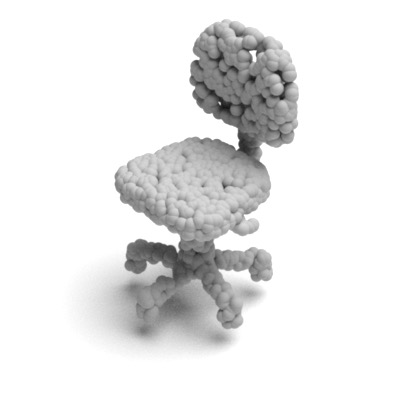}%
  \includegraphics[width=\sizea]{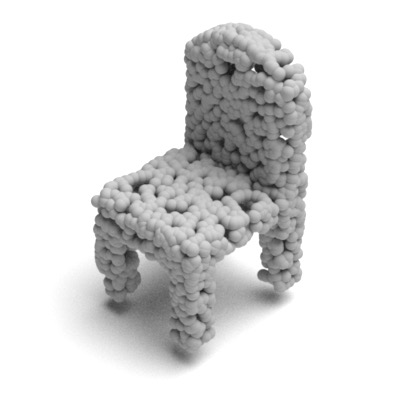}%
  \includegraphics[width=\sizea]{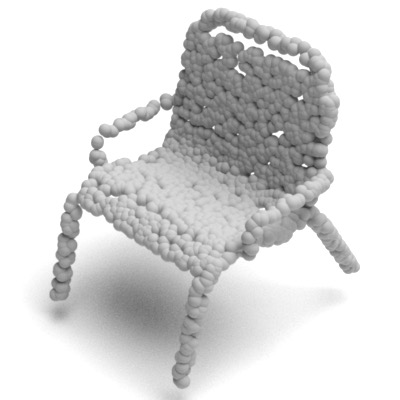}\\
  \includegraphics[width=\sizea]{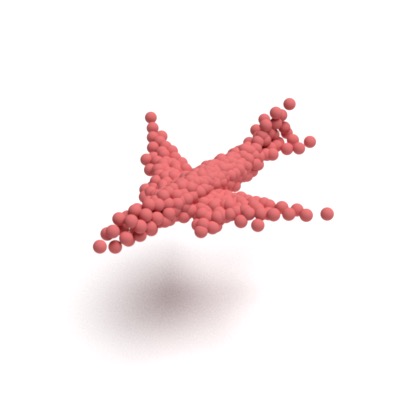}%
  \includegraphics[width=\sizea]{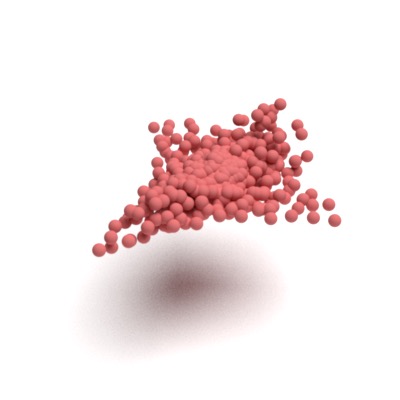}%
  \includegraphics[width=\sizea]{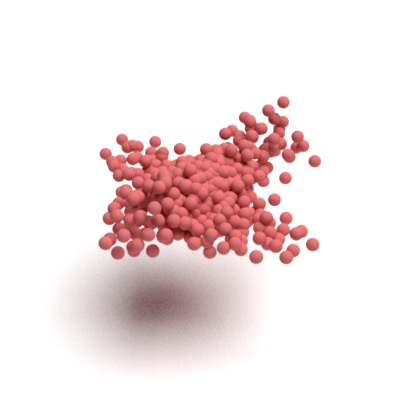}%
  \includegraphics[width=\sizea]{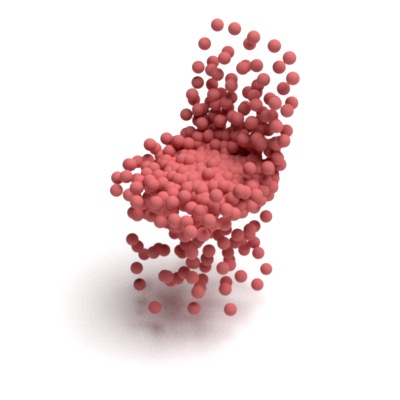}%
  \includegraphics[width=\sizea]{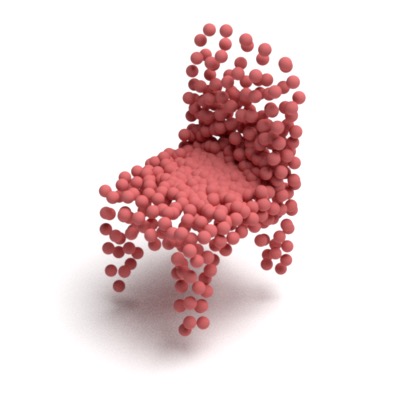}%
  \includegraphics[width=\sizea]{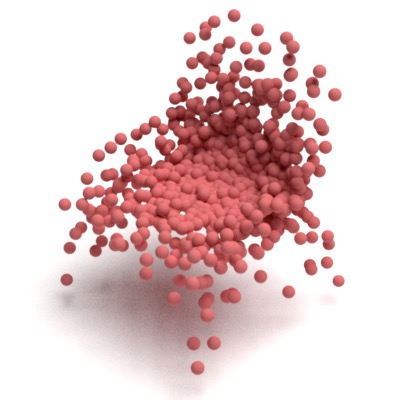}\\
  \includegraphics[width=\sizea]{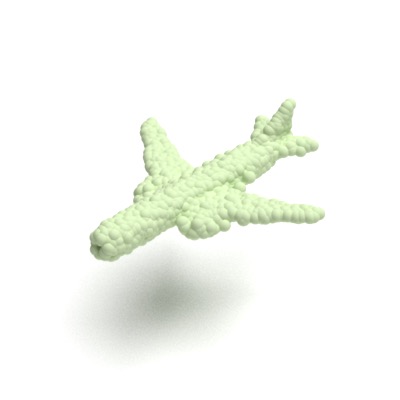}%
  \includegraphics[width=\sizea]{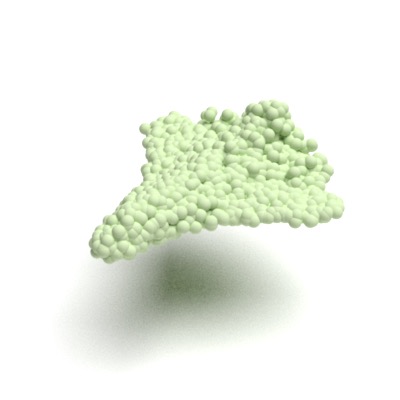}%
  \includegraphics[width=\sizea]{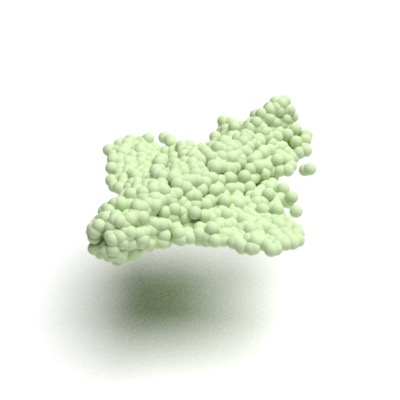}%
  \includegraphics[width=\sizea]{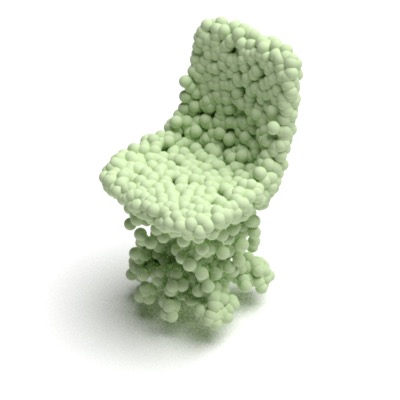}%
  \includegraphics[width=\sizea]{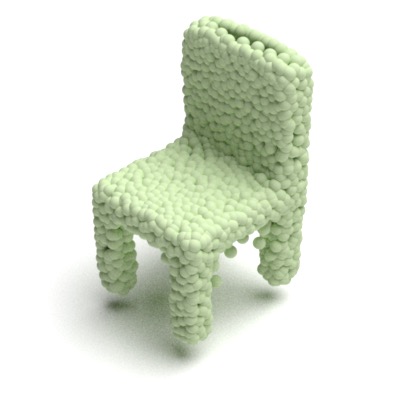}%
  \includegraphics[width=\sizea]{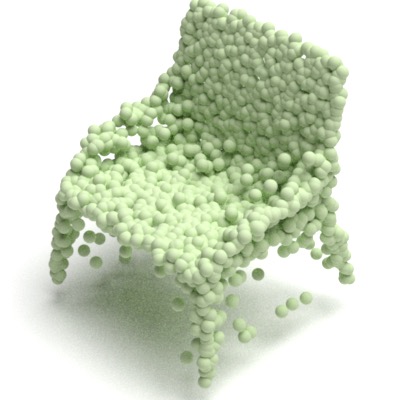}\\
  \includegraphics[width=\sizea]{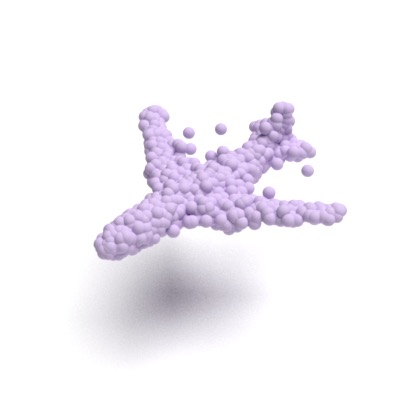}%
  \includegraphics[width=\sizea]{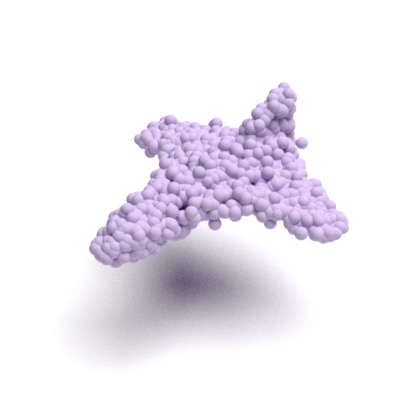}%
  \includegraphics[width=\sizea]{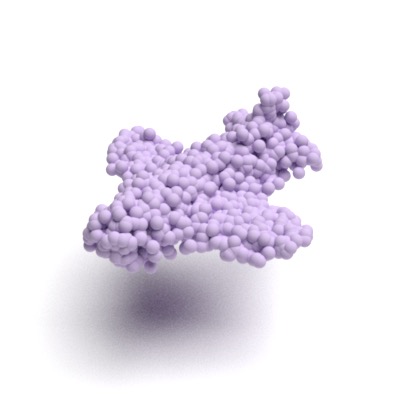}%
  \includegraphics[width=\sizea]{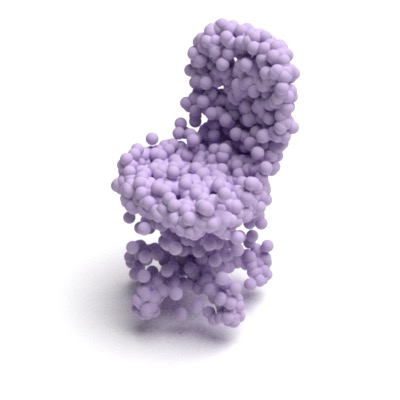}%
  \includegraphics[width=\sizea]{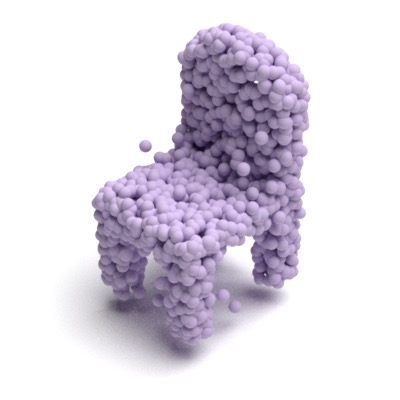}%
  \includegraphics[width=\sizea]{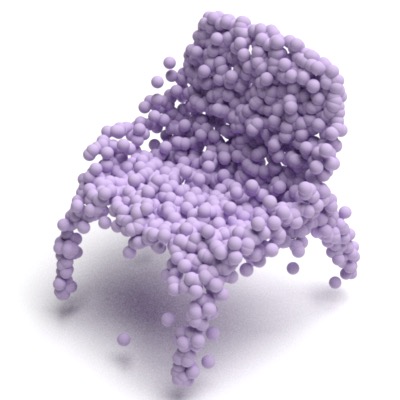}
  \caption{
    Qualitative results of autoencoders trained on single class using different loss functions.
    From top to bottom: input point clouds, \textsc{CD-AE} (red), \textsc{EMD-AE} (green) and \textsc{SSW-AE} (magenta).
    Compared to our models, \textsc{CD-AE} fails to reconstruct properly most of the 3D shapes.
  }
  \label{fig:recon}
\end{figure*}
\vspace{0.5 em}
\noindent
\textbf{Transfer learning.}
We further evaluate the performance of the autoencoders by using their latent vectors as features for classification.
Particularly, for an input 3D shape, we feed its point cloud into an autoencoder and extract the corresponding latent vector.
This vector is then classified by a classifier trained on the de-facto 3D classification benchmark of ModelNet40~\cite{wu2015modelnet}.
The architecture of the classifier is shown in Figure~\ref{fig:network}.
The input is a 256-dimension feature vector and the output is a 40-dimension vector representing the prediction scores of 40 classes in ModelNet40.
We train our networks for 500 epochs with a batch size of 256.
We use an SGD optimizer with $0.001$ learning rate, $0.9$ momentum, and $0.005$ weight decay.
The classification results are shown in Table~\ref{tab:classification}.
As can be seen, autoencoders trained with sliced Wasserstein distance outperformed both Chamfer and EMD. 
We further investigate performance of classifiers in case point clouds in ModelNet40 are perturbed by noise. We find that features learned with SWD are the most robust to noise, outperforming both CD and EMD by about 3\% of accuracy. Details can be found in the supplementary material.

\begin{table*}[t]
  \centering
  \begin{tabularx}{\textwidth}{l*{7}{Y}}
    \toprule
    &  \multirow{2}{*}{JSD ($\downarrow$)} & \multicolumn{2}{c}{MMD ($\downarrow$)} & \multicolumn{2}{c}{COV (\%, $\uparrow$)} & \multicolumn{2}{c}{1-NNA (\%, $\downarrow$)} \\
    \cmidrule(lr){3-4} \cmidrule(lr){5-6} \cmidrule(lr){7-8}
    Method                 &               & CD            & EMD            & CD             & EMD            & CD             & EMD            \\
    \midrule
    \textsc{CD-AE}         & 38.97         & 0.65          & 23.44          & 31.91          & 5.47           & 86.63          & 100.00         \\
    \textsc{EMD-AE}        & 3.73          & \textbf{0.61} & \textbf{10.44} & \textbf{35.75} & 35.75          & \textbf{86.34} & \textbf{87.96} \\
    \textsc{SSW-AE} (ours) & \textbf{3.24} & 0.79          & 11.22          & 28.51          & \textbf{37.96} & 91.43          & 91.80          \\
    \bottomrule
  \end{tabularx}
  \caption{
    Quantitative results of point cloud generation task on the chair category of ShapeNet.
    $\uparrow$: the higher the better, $\downarrow$: the lower the better.
    JSD, MMD-CD, and MMD-EMD scores are all multiplied by $10^2$.
  }
  \label{tab:pointgen}
\end{table*}

\vspace{0.5 em}
\noindent
\textbf{Point cloud generation.}
Next, we evaluate our method on the point cloud generation task.
Following~\cite{achlioptas2018pointnetae}, we split the chair category of ShapeNet into train/validation/test sets in a 85/5/10 ratio.
We train the autoencoders using different distance metrics for $10^4$ epochs.
After that, we train a generator on the latent space of an autoencoder, same as~\cite{achlioptas2018pointnetae}.
Our generators, parameterized by a multi-layer perceptron, learn to map a 64-dimensional vector drawn from a normal distribution $\mathcal{N}(0, \mathbb{I}_{64})$ to a latent code learned by an autoencoder $\mathbb{P}_{\text{latent}}$, where $\mathbb{I}_{64}$ is the $64 \times 64$ identity matrix.
We train the generators by minimizing the optimal transport distance between the generated and ground truth latent codes.
We report the quantitative in Table~\ref{tab:pointgen}.
We use the same evaluation metrics as proposed by~\cite{yang2019pointflow}. Qualitative results and results for MSW-AE and ASW-AE can be found in the supplementary.

\begin{table}[t]
  \centering
  \begin{tabular}{lcccc}
    \toprule
            & \textsc{CD-AE} & \textsc{EMD-AE} & \textsc{SSW-AE} & \textsc{CZK}~\cite{choi2015robust} \\
    \midrule
    home1   & 59.4           & \textbf{60.4}  & \textbf{60.4}            & 63.2                               \\
    home2   & 47.2           & 46.5  & \textbf{47.8}            & 40.3                               \\
    hotel1  & 62.6           & 62.1           & \textbf{69.8}   & 64.3                               \\
    hotel2  & 43.6           & 44.9           & 48.7            & \textbf{66.7}                      \\
    hotel3  & 46.2           & 34.6           & \textbf{65.4}   & 57.7                               \\
    kitchen & 58.4           & 57.0           & \textbf{62.6}   & 49.9                               \\
    lab     & 42.2           & 46.7           & \textbf{48.9}   & 37.8                               \\
    study   & 50.4           & 50.0           & \textbf{55.6}   & 54.7                               \\
    \midrule
    Average & 51.3           & 50.3           & \textbf{57.4}   & 54.3                               \\
    \bottomrule
  \end{tabular}
  \caption{
    3D registration results (recall) on the 3DMatch benchmark.
    We compare the models that trained with sliced Wasserstein (SW-AE) and squared sliced Wasserstein (SSW-AE) against Chamfer discrepancy (CD-AE) and a geometric-based approach (CZK). More details can be found in the suppplementary.
  }
  \label{tab:registration}
\end{table}


\vspace{0.5 em}
\noindent
\textbf{3D point cloud registration.}
Finally, we consider the problem of 3D point cloud registration.
In this problem, we need to estimate a rigid transformation between two 3D point clouds.
We follow the same setup as~\cite{pham2020lcd} and use the autoencoders for local feature extraction.
Evaluation is performed on the standard 3DMatch benchmark~\cite{zeng20173dmatch}.
We also compare our models against \textsc{CZK}~\cite{choi2015robust}, a method that used geometric features.
The final results are shown in Table~\ref{tab:registration}.
Our methods outperform other models by a good margin. Results for MSW-AE and ASW-AE can be found in the supplementary.


\noindent
\textbf{Runtime performance.}
We report the training time per iteration when training the autoencoder using Sliced Wasserstein with fixed number of slices, Chamfer and approximated EMD.
We train over $10^4$ iterations with a batch size of 128 and report the average timing.
For Chamfer and EMD, we use the implementation from~\cite{yang2019pointflow}.
Otherwise, we use our Python implementation.
Table~\ref{tab:runtime} shows the runtime performance of different metrics.
Our proposed sliced Wasserstein distance is as fast as Chamfer discrepancy, while being more accurate.
Compared to EMD, sliced Wasserstein distance is almost three times faster.
\begin{table}[t]
  \centering
  \begin{tabular}{lc}
    \toprule
    Distance & Runtime (ms) \\
    \midrule
    EMD     & 385 \\
    CD      & 120 \\
    SWD     & 138 \\
    \bottomrule
  \end{tabular}
  \caption{
    Training time per iteration in milliseconds of different distance functions.
    We compare the sliced Wasserstein distance (SWD) against Chamfer discrepancy (CD) and approximated EMD.
  }
  \label{tab:runtime}
\end{table}

\vspace{0.5 em}
\noindent
\textbf{Convergence rate.}
Our proposed sliced Wasserstein distance also has better convergence rate compared to Chamfer and EMD.
To demonstrate this point, we calculate the error during training using exact EMD.
Results in Table~\ref{tab:convergence} show that all of the SW variants converge much faster than Chamfer and EMD.
\begin{table}[t]
  \centering
  \begin{tabular}{lccc}
    \toprule
    Method          & Epoch 50         & Epoch 150        & Epoch 300        \\
    \midrule
    \textsc{CD-AE}  & 0.268            & 0.275            & 0.314 \\
    \textsc{EMD-AE} & 0.171            & 0.152            & 0.144 \\
    \textsc{SSW-AE} & $\mathbf{0.106}$   & $\mathbf{0.097}$   & $\mathbf{0.091}$ \\
    \textsc{MSW-AE} & 0.110            & 0.099            & 0.093 \\
    \textsc{ASW-AE} & 0.109            & 0.098            & 0.092 \\
    
    \bottomrule
  \end{tabular}
  \caption{
    Convergence rate of different distance metrics during training.
    We report the exact EMD errors at epoch 50, 150 and 300.
    Sliced Wasserstein distance has the best convergence rate.
  }
  \label{tab:convergence}
\end{table}

\noindent\textbf{Different architecture.}
We further performed experiments on point cloud capsule networks~\cite{zhao2019pointcapsule} to show that our method is agnostic to network architecture. As illustrated in Table~\ref{tab:recon-pcn}, SWD is agnostic to network architecture we tested, i.e., SWD also works well for the point cloud capsule network~\cite{zhao2019pointcapsule} (PCN), which has a very different design from the original PointNet. We trained PCN on ShapeNetCore55 dataset and tested on ModelNet40. As can be seen, using SWD can result in a slight performance improvement compared to Chamfer distance in the reconstruction task. Using SWD with PCN leads to a significant improvement of almost 2\% in the classification task.

\begin{table}[t]
  \small
  \centering
  \begin{tabular}{l|ccc|c}
    \toprule
    Method                 & CD             & SWD            & EMD     & Accuracy          \\
    \midrule
    \textsc{PCN-SSW}    & 0.006          & 0.761          & 0.084     & 88.78     \\
    \textsc{PCN-CD}     & 0.003          & 3.035          & 0.156     & 88.45     \\
    \bottomrule
  \end{tabular}
  \caption{
    Quantitative measurements of the discrepancy between the input point clouds and their reconstructed versions on ModelNet40. The last column is the classification accuracy on ModelNet40. 
  }
  \label{tab:recon-pcn}
  \vspace{-0.3cm}
\end{table}

\section{Conclusion}
In the paper, we propose using sliced Wasserstein distance 
for learning representation of 3D point clouds. We theoretically demonstrate that the sliced Wasserstein distance is equivalent to EMD while its computational complexity is comparable to Chamfer divergence. Therefore, it possesses both the statistical and computational benefits of EMD and Chamfer divergence, respectively. We also propose a new algorithm to approximate sliced Wasserstein distance between two given point clouds so that the estimation is close enough to the true value. Empirically, we show that the latent codes of the autoencoders learned using sliced Wasserstein distance are more useful for various downstream tasks than those learned using the Chamfer divergence and EMD.\\ 
\textbf{Acknowledgment.} Tam Le acknowledges the support of JSPS KAKENHI Grant number 20K19873.

{\small
  \bibliographystyle{ieee_fullname}
  \bibliography{references}
}

\newpage
\appendix

\newpage

\section*{Appendix}

In this supplemental material, we further test the robustness of the proposed metrics (Section~\ref{sec:robustness}), and then report numerical results on more metrics including adaptive-sliced Wasserstein (\textsc{ASW}), max sliced Wasserstein (\textsc{MSW}) and generalized sliced Wasserstein (\textsc{GSW}) (Section~\ref{sec:gsw}). We also report additional results on the point cloud registration (Section~\ref{sec:registration}) and the point cloud generation task (Section~\ref{sec:generation}). 
Additionally, we also provide an evaluation of the number of slices used in the sliced Wasserstein (SSW) on the reconstruction, classification, and registration task.

\section{Robustness}
\label{sec:robustness}
We conduct an experiment to compare robustness between Chamfer, EMD, and SWD. Particularly, we train autoencoder using Chamfer, EMD, and SWD respectively on ShapeNet, with point coordinates in $[-1, 1]$. At test time, we use ModelNet40, and the point clouds are perturbed by Gaussian noise $N(0,\sigma^2)$ with $\sigma \in[0.01, 0.05]$. 
We use the autoencoder to extract features from noisy point clouds and then input to learn a classifier. 
Figure~\ref{fig:robustness} shows the performance of the classifier with increasing standard deviation values, where the experiments are carried out three times and then taken average. The solid lines demonstrate the case where we train the autoencoder with clean point clouds, while the dashed lines demonstrate the case where we train with noisy point clouds, i.e., we perturb ShapeNet in the same way as we do with ModelNet40. In both cases, Figure~\ref{fig:robustness} shows that features learned with SWD are the most robust to noise, outperforming both CD and EMD by about 3\% of accuracy. 
We also found that CD is less robust than EMD.
\begin{figure}[h]
    \small
  \centering
    \includegraphics[width=\linewidth]{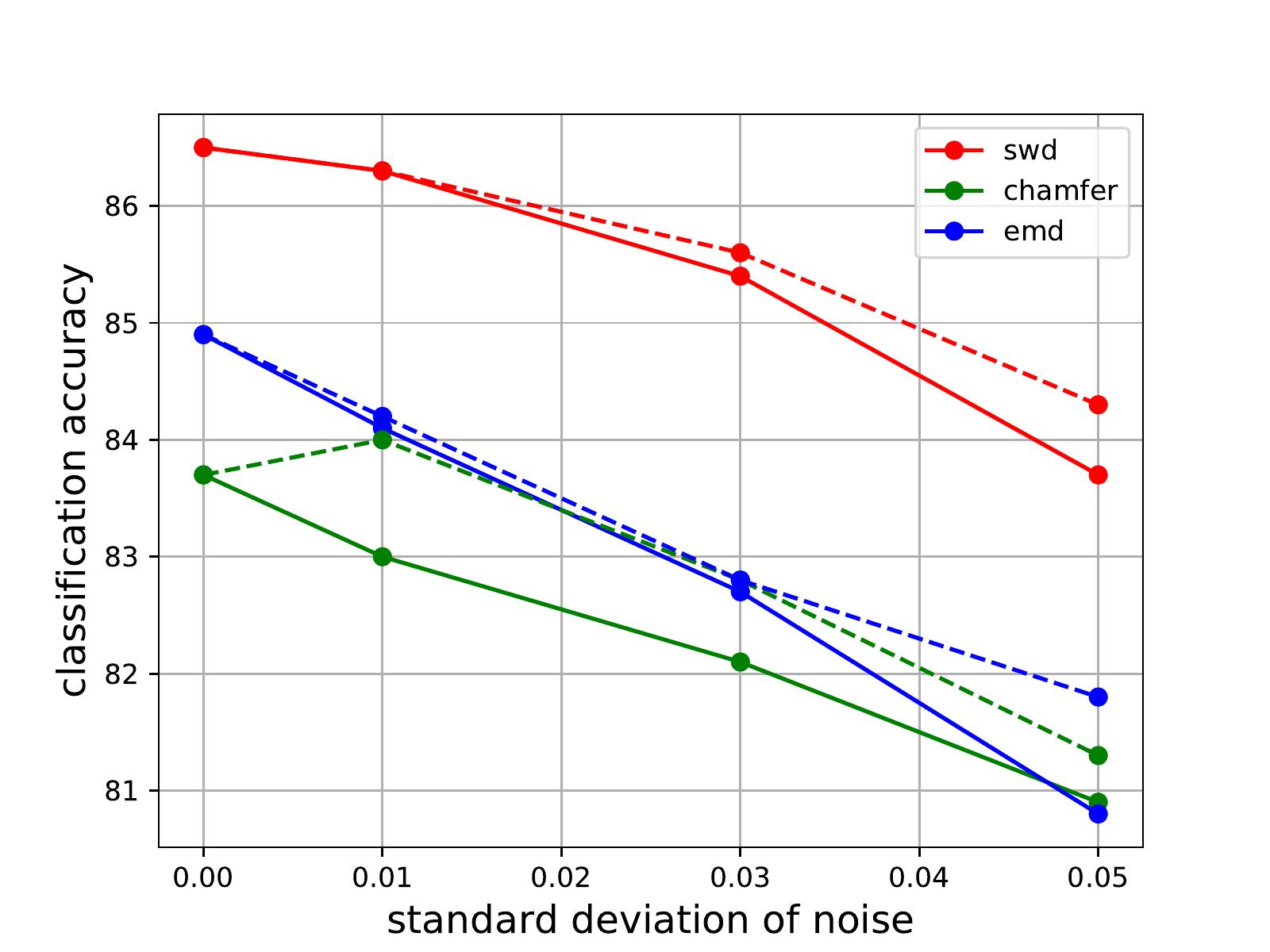}
  \caption{Classification accuracy on ModelNet40 with noisy data.}
  \label{fig:robustness}
  \vspace{-0.1in}
\end{figure}

\subsection{Performance with respect to batch size}
We provide an experiment with batch size in Table~\ref{tab:batchsize}, which shows negligible change in the Chamfer discrepancy between the input and reconstructed point clouds in ModelNet40 across batch sizes. 
\begin{table}[ht]
\centering
\begin{tabular}{cc|c|c|c|c|l}
\cline{3-5}
& & \multicolumn{3}{ c| }{Batch size} \\ \cline{3-5}
& & 32 & 128 & 256 \\ \cline{1-5}
\multicolumn{1}{ |c  }{\multirow{3}{*}{Model} } &

\multicolumn{1}{ |c| }{\textsc{SWD-AE}} & \textbf{0.006} & \textbf{0.007} & \textbf{0.008} \\ \cline{2-5}

\multicolumn{1}{ |c  }{}                         &

\multicolumn{1}{ |c| }{\textsc{CD-AE}} & 0.012 & 0.014 & 0.014 \\ \cline{2-5}

\multicolumn{1}{ |c  }{}                         &

\multicolumn{1}{ |c| }{\textsc{EMD-AE}} & 0.012 & 0.014 & 0.013 \\ \cline{1-5}
\end{tabular}
\caption{
    Average of the discrepancy between the input point clouds and their reconstructed versions on ModelNet40 with different batch sizes.
}
\label{tab:batchsize}
\end{table}

\section{Generalized sliced Wasserstein distance}
\label{sec:gsw}
First, we recall briefly the definition of generalized sliced Wasserstein distance~\cite{gsw_nips19}. Generalized sliced-Wasserstein distance (GSW) extends the sliced-Wasserstein distance by replacing the inner product $\theta^{\top}x$ with a defining function $g(\theta, x)$ (cf. Assumptions H1-H4 in~\cite{gsw_nips19} for the definition of defining function). Denote $\pi_{g_{\theta}}\sharp\mu$ the pushforward measure of $\mu$ through the mapping $g_{\theta}$ where $g_\theta(x):=g(\theta, x)$ for all $x$. Then, for $p\geq1$, the \emph{GSW} is given by
\begin{equation}
    GSW_p(\mu, \nu) : = \Big(\int_{\Omega_\theta}W^p_p(\pi_{g_\theta}\sharp\mu, \pi_{g_\theta}\sharp\nu)\Big)^{1/p}
\end{equation}
where $\Omega_\theta$ is the compact set of feasible parameters.\\
In our experiments, $\Omega_\theta:=\mathbb{S}^{2}$ and $g(x, \theta) := ||x-\theta||_2$. To estimate GSW, we use Monte Carlo scheme as follows:
\begin{align}
    GSW_p(\mu,\nu) \approx \Big(\frac{1}{N} \sum_{i=1}^N W_p^p\big(\pi_{g_{\theta_i}} \sharp\mu, \pi_{g_{\theta_i}} \sharp \nu\big) \Big)^{\frac{1}{p}}. \label{eq:approximation_GSW}
\end{align}
where we set $N:=100$.
In Table~\ref{tab:recon-cls}, we provide numerical results for \textsc{GSW} on reconstruction and classification tasks.
\textcolor{black}{As we can see, \textsc{GSW} is slightly better than \textsc{SW} and \textsc{MSW} in reconstruction task, while \textsc{SW} is slightly better than other variants in classification task.}

\begin{table}[ht]
  \small
  \centering
  \begin{tabular}{l|ccc|c}
    \toprule
    Method                 & CD             & SWD            & EMD     & Accuracy (\% ) \\
    \midrule
    \textsc{CD-AE}         & 0.014          & 6.738          & 0.314    &  83.9     \\
    \textsc{EMD-AE}        & 0.014          & 2.295          & 0.114    &  84.4    \\
    \textsc{SSW-AE} (ours) & 0.007 & 0.831 & 0.091 & \textbf{86.8} \\
    \textsc{ASW-AE} (ours) & 0.007 & 0.854 & 0.092 & \textbf{86.8} \\
    \textsc{MSW-AE} (ours) & 0.007 & 0.865 & 0.093 & 86.5 \\
    \textsc{GSW-AE} (ours) & \textbf{0.006} & \textbf{0.816} & \textbf{0.090} & 85.8 \\
    \bottomrule
  \end{tabular}
  \caption{
    Quantitative measurements of the discrepancy between the input point clouds and their reconstructed versions on ModelNet40. The last column is the classification accuracy on ModelNet40. 
  }
  \label{tab:recon-cls}
  \vspace{-0.3cm}
\end{table}

\begin{table}[t]
  \centering
  \begin{tabular}{l|ccc|c}
    \toprule
    Method                 & CD             & SWD            & EMD     & Accuracy(\%)       \\
    \midrule
    \textsc{CD-AE}         & 0.014          & 6.738          & 0.314    & 83.9     \\
    \textsc{EMD-AE}        & 0.014          & 2.295          & 0.114    & 84.4     \\
    \midrule
    \textsc{SSW1-AE}       & \textbf{0.007} & 0.901 & 0.094 & 86.5\\
    \textsc{SSW2-AE}       & \textbf{0.007} & 0.865 & 0.093 & 86.5\\
    \textsc{SSW5-AE}       & \textbf{0.007} & 0.829 & \textbf{0.091} & 86.7\\
    \textsc{SSW10-AE}      & \textbf{0.007} & \textbf{0.812} & \textbf{0.091} & \textbf{86.8}\\
    \textsc{SSW50-AE}      & \textbf{0.007} & 0.849 & 0.092 & \textbf{86.8}\\
    \textsc{SSW100-AE}     & \textbf{0.007} & 0.831 & \textbf{0.091} & \textbf{86.8}\\

    \bottomrule
  \end{tabular}
  \caption{
    Quantitative measurements of the discrepancy between the input point clouds and their reconstructed versions on ModelNet40. The last column is the classification accuracy on ModelNet40. 
  }
  \label{tab:recon-full}
\end{table}

\subsection{Effect of the number of slices}
We measure the effect of varying the number of slices when computing sliced Wasserstein distance using Monte Carlo scheme. 
We denote $\text{SSW}n-\text{AE}$ the auto-encoders trained using the sliced Wasserstein distance estimated by Monte Carlo estimation with $n$ projections.
We provide quantitative results for reconstruction and classification tasks in Table~\ref{tab:recon-full}. 
Table~\ref{tab:recon-full} shows that increasing the number of slices in Monte Carlo estimation does not affect performance much in reconstruction and classification tasks.

\section{Point cloud registration}
\label{sec:registration}
In Table~\ref{tab:registration_supp}, we provide quantitative results as discussed in section Point cloud registration on page 7 of the main paper. \textcolor{black}{Table~\ref{tab:registration_supp} shows that \textsc{SSW} archives the best recall on average.}
\begin{table}[ht]
  \centering
  \begin{tabular}{lccccc}
    \toprule
            & \textsc{ASW-AE} & \textsc{MSW-AE} & \textsc{GSW-AE} & \textsc{SSW-AE}\\
    \midrule
    home1   & \textbf{63.2} & \textbf{63.2} & 62.3 & 60.4
        \\
    home2   & 49.1 & 49.1   & \textbf{52.8}&47.8
        \\
    hotel1  & 67.6 & 66.5 & 68.1 & \textbf{69.8}
        \\
    hotel2  & 46.2 & \textbf{48.7}  &46.2&\textbf{48.7}
        \\
    hotel3  & 57.7 & 53.8 & 57.7 & \textbf{65.4}
        \\
    kitchen & \textbf{64.1} &63.3  & 63.3 & 62.6
        \\
    lab     & 44.4 &44.4  & 46.7 & \textbf{48.9}
        \\
    study   & 57.7           &55.6   & \textbf{58.5} & 55.6
        \\
    \midrule
    Average & 56.3           &55.6    & 57.0 & \textbf{57.4}
        \\
    \bottomrule
  \end{tabular}
  \caption{
    3D registration results (recall) on the 3DMatch benchmark.
  }
  \label{tab:registration_supp}
\end{table}

\begin{table*}[ht]
	\centering
	\begin{tabular}{lccccccc}
		\toprule
		& \textsc{SSW1-AE} & \textsc{SSW5-AE} & \textsc{SSW10-AE} & \textsc{SSW50-AE} &\textsc{SSW100-AE} & \textsc{EMD-AE} & \textsc{CD-AE} \\
		\midrule
		home1   &62.3 & 63.2 & 61.3 & 63.2 &60.4 & 60.4 & 59.4\\
		home2   &49.7 & 48.4 & 49.1 & 50.9 &47.8 & 46.5 & 47.2\\
		hotel1  &65.9 & 68.7 & 65.9 & 68.1 &69.8 & 62.1 & 62.6\\
		hotel2  &50.0 & 43.6 & 43.6 & 47.4 &48.7 & 44.9 & 43.6\\
		hotel3  &50.0 & 57.7 & 53.8 & 65.4 &65.4 & 34.6 & 46.2\\
		kitchen &64.4 & 62.6 & 63.7 & 62.1 &62.6 & 57.0 & 58.4\\
		lab     &42.2 & 40.0 & 46.7 & 48.9 &48.9 & 46.7 & 42.2\\
		study   &56.4 & 56.0 & 56.0 & 55.6 &55.6 & 50.0 & 50.4\\
		\midrule
		Average &55.1 & 55.0 & 55.0 & \textbf{57.7} & \underline{57.4} & 50.3 & 51.3\\
		\bottomrule
	\end{tabular}
	\caption{
		Varying number of slices for the 3D registration task. The best scores are highlighted in bold. The second best scores are underlined.
	}
	\label{tab:registration_supp-full}
\end{table*}

\begin{table*}[ht]
  \centering
  \begin{tabularx}{\textwidth}{l*{7}{Y}}
    \toprule
    &  \multirow{2}{*}{JSD ($\downarrow$)} & \multicolumn{2}{c}{MMD ($\downarrow$)} & \multicolumn{2}{c}{COV (\%, $\uparrow$)} & \multicolumn{2}{c}{1-NNA (\%, $\downarrow$)} \\
    \cmidrule(lr){3-4} \cmidrule(lr){5-6} \cmidrule(lr){7-8}
    Method                 &               & CD            & EMD            & CD             & EMD            & CD             & EMD            \\
    \midrule
    \textsc{SSW-AE} & 3.24 & 0.79 & 11.22 & 28.51 & 37.96 & 91.43 & 91.80 \\
    \textsc{ASW-AE} & 3.58 & 0.73 & 10.65 & 31.76 & 36.48 & 93.57 & 93.94 \\
    \textsc{MSW-AE} & 3.96 & \textbf{0.59} & \textbf{9.64} & \textbf{35.89} & \textbf{40.47} & \textbf{89.73} & \textbf{89.29} \\
    \textsc{GSW-AE} & \textbf{3.06} & 0.76 & 10.98 & 30.13 & 37.52 & 91.21 & 91.65 \\
    \bottomrule
  \end{tabularx}
  \caption{
    Quantitative results of point cloud generation task on the chair category of ShapeNet.
    $\uparrow$: the higher the better, $\downarrow$: the lower the better.
    JSD, MMD-CD, and MMD-EMD scores are all multiplied by $10^2$.
  }
  \label{tab:pointgen_supp}
\end{table*}

\begin{figure*}[t]
  \centering
  \newcommand{\sizea}{0.15\linewidth}
  \includegraphics[width=\sizea]{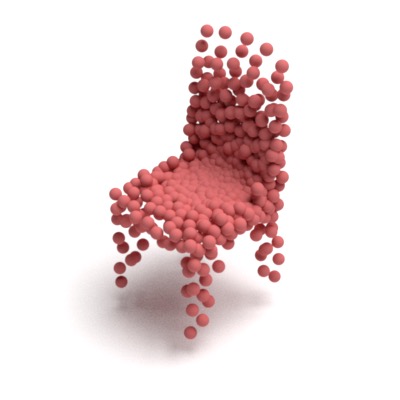}%
  \includegraphics[width=\sizea]{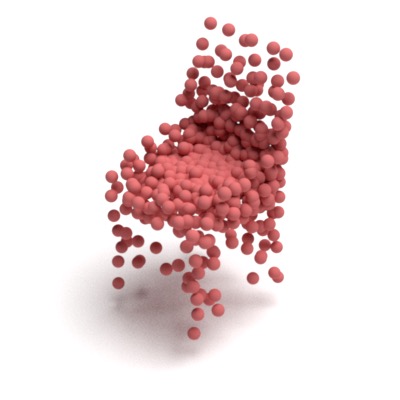}%
  \includegraphics[width=\sizea]{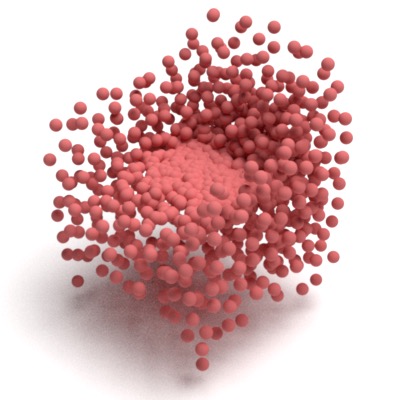}%
  \includegraphics[width=\sizea]{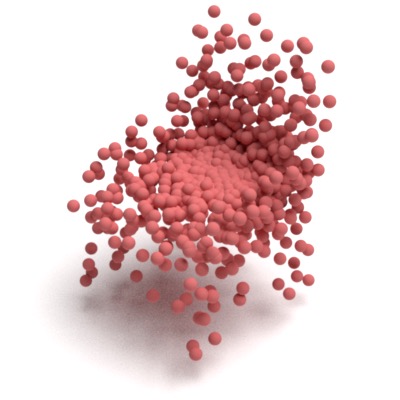}%
  \includegraphics[width=\sizea]{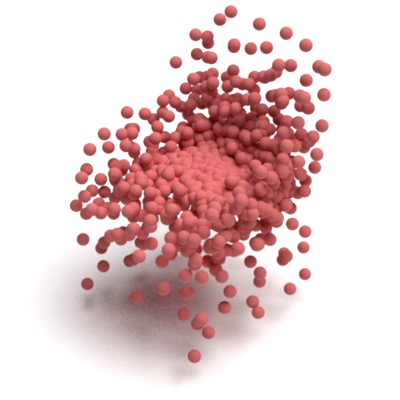}%
  \includegraphics[width=\sizea]{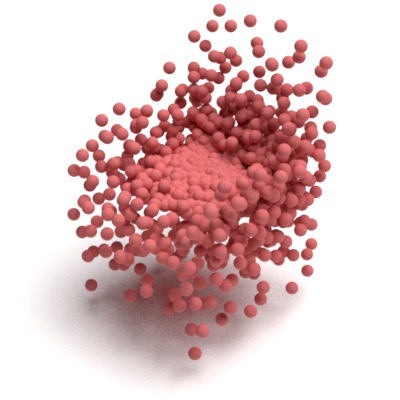}\\
  \includegraphics[width=\sizea]{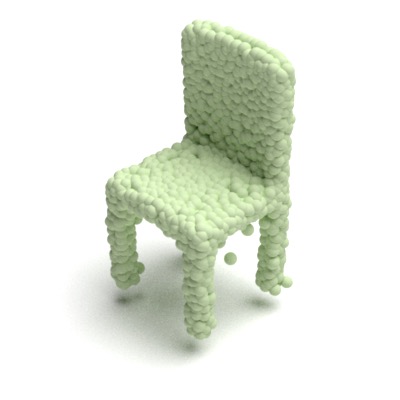}%
  \includegraphics[width=\sizea]{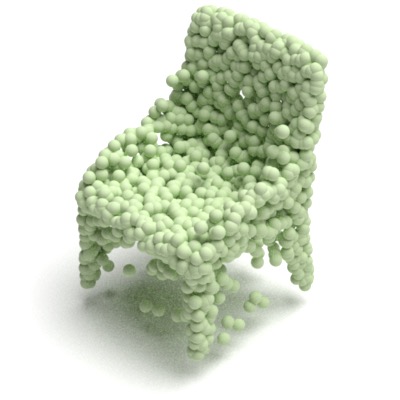}%
  \includegraphics[width=\sizea]{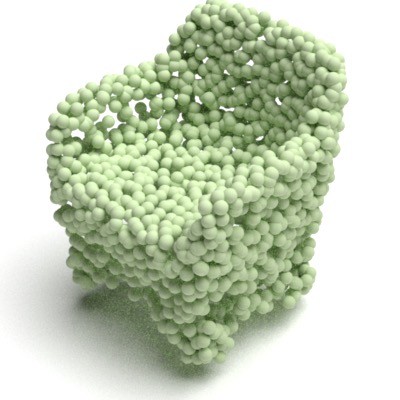}%
  \includegraphics[width=\sizea]{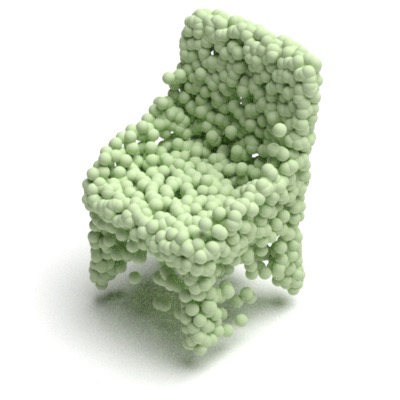}%
  \includegraphics[width=\sizea]{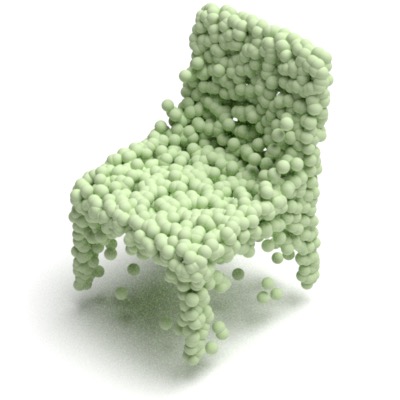}%
  \includegraphics[width=\sizea]{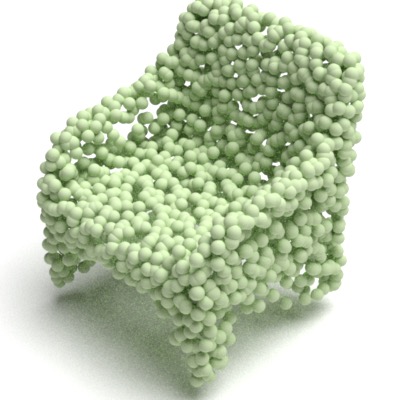}\\

  \includegraphics[width=\sizea]{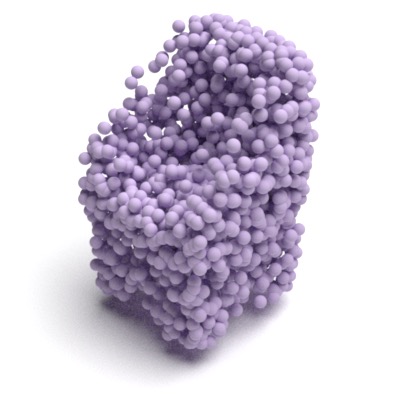}%
  \includegraphics[width=\sizea]{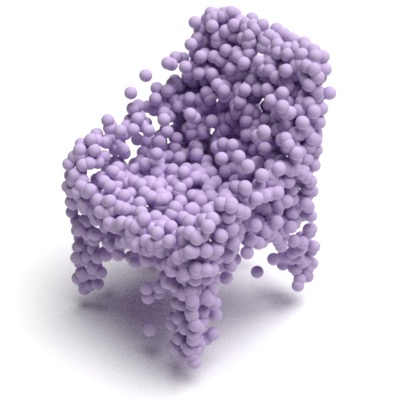}%
  \includegraphics[width=\sizea]{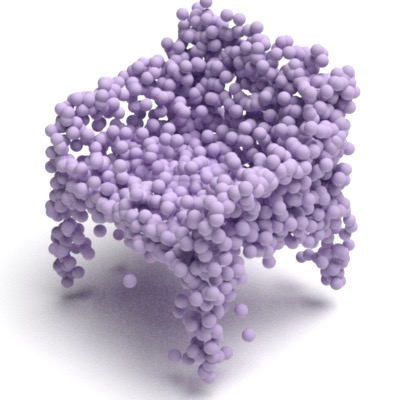}%
  \includegraphics[width=\sizea]{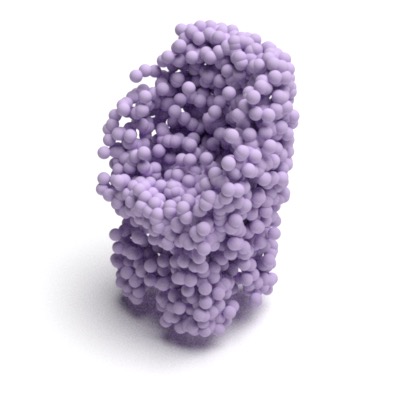}%
  \includegraphics[width=\sizea]{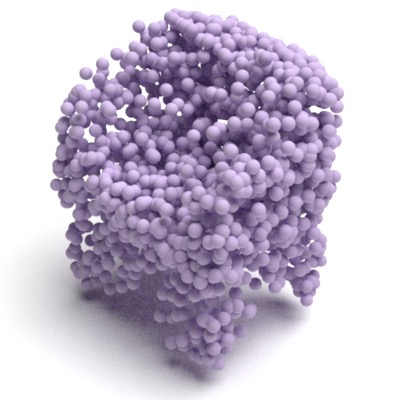}%
  \includegraphics[width=\sizea]{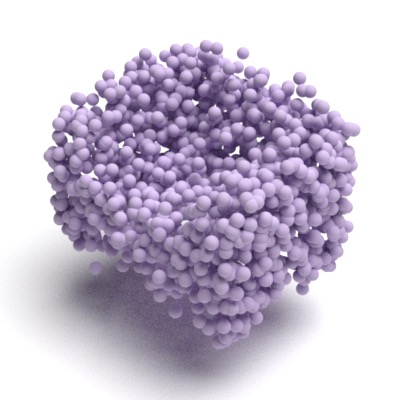}\\
    \includegraphics[width=\sizea]{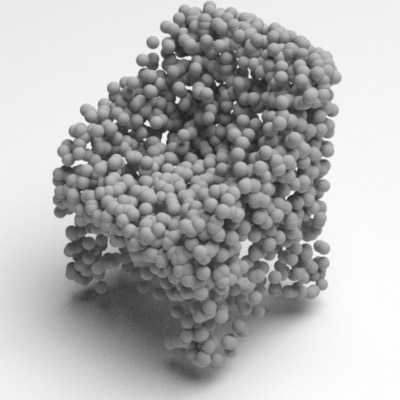}%
  \includegraphics[width=\sizea]{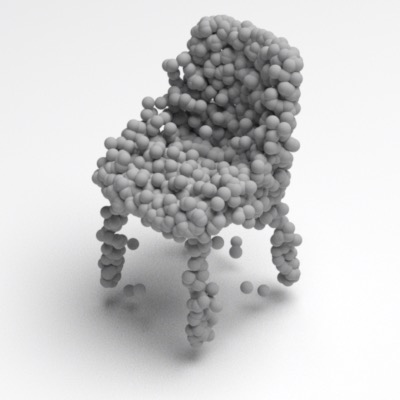}%
  \includegraphics[width=\sizea]{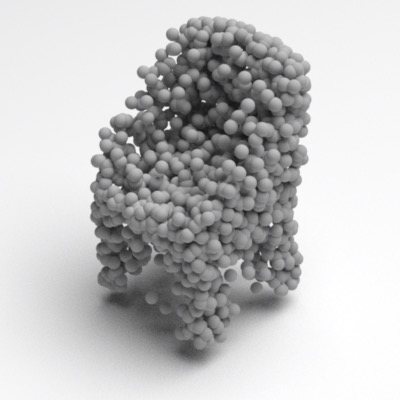}%
  \includegraphics[width=\sizea]{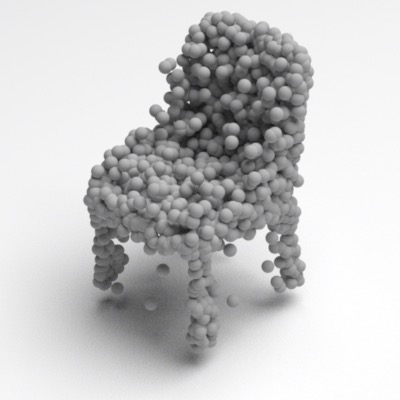}%
  \includegraphics[width=\sizea]{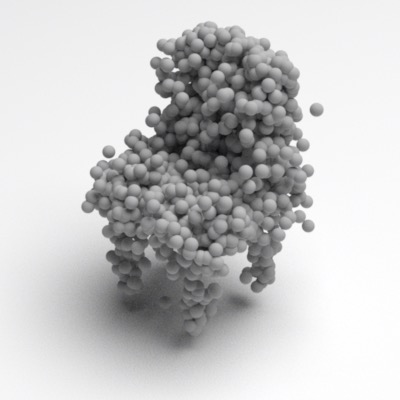}%
  \includegraphics[width=\sizea]{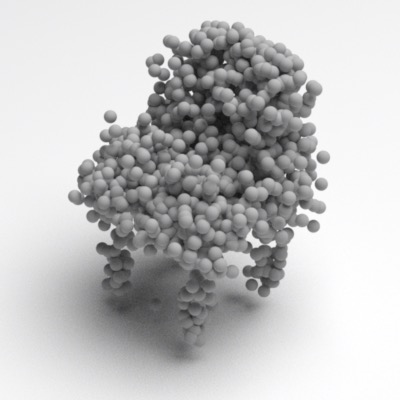}\\
  \includegraphics[width=\sizea]{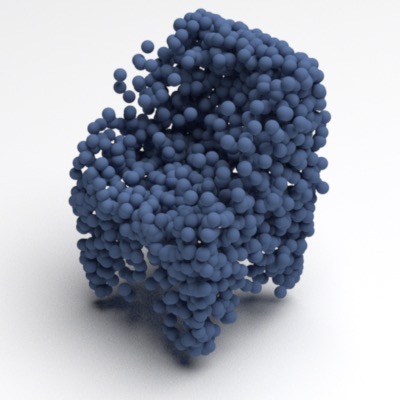}%
  \includegraphics[width=\sizea]{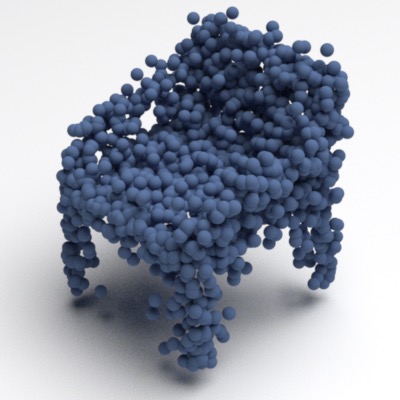}%
  \includegraphics[width=\sizea]{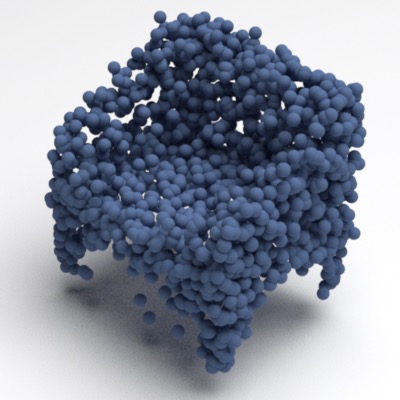}%
  \includegraphics[width=\sizea]{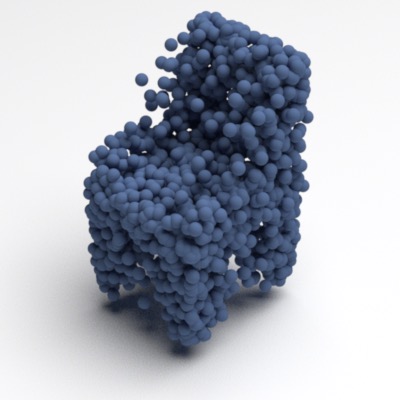}%
  \includegraphics[width=\sizea]{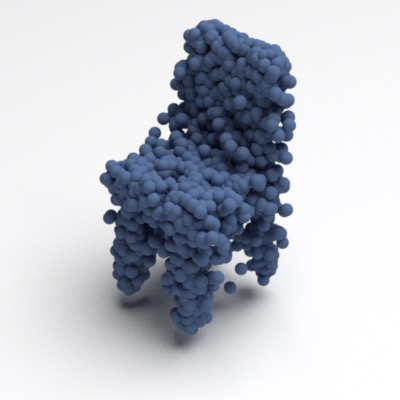}%
  \includegraphics[width=\sizea]{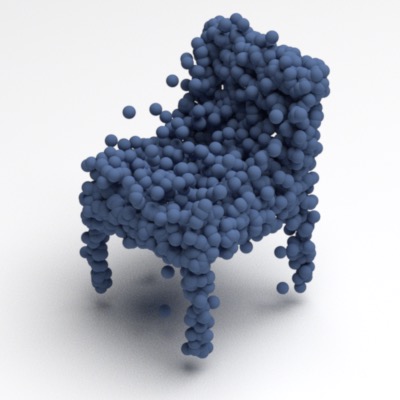}\\
     \includegraphics[width=\sizea]{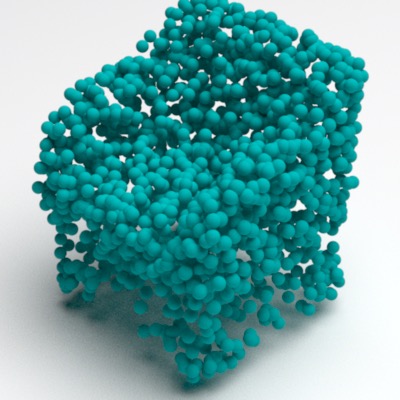}%
  \includegraphics[width=\sizea]{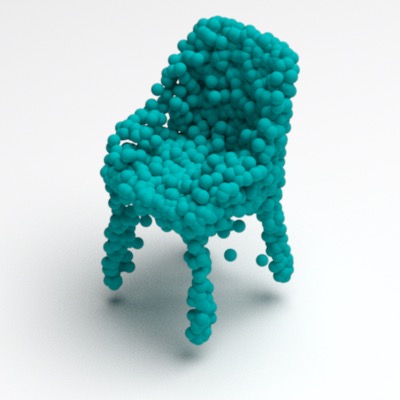}%
  \includegraphics[width=\sizea]{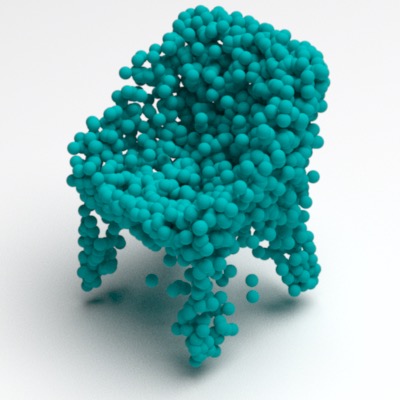}%
  \includegraphics[width=\sizea]{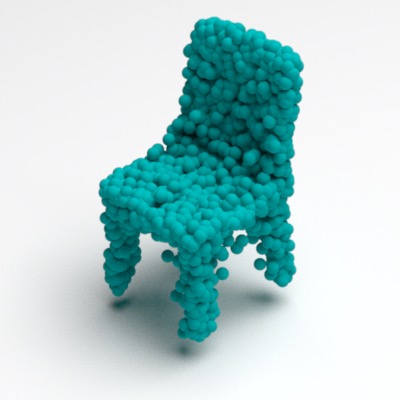}%
  \includegraphics[width=\sizea]{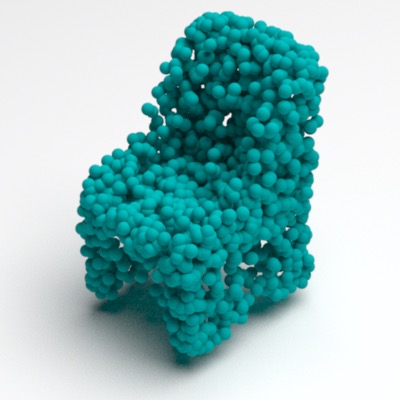}%
  \includegraphics[width=\sizea]{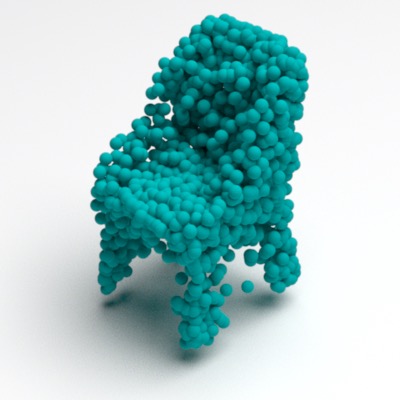}
  \caption{
    Point cloud generation results of the trained autoencoders on the chair category of ShapeNet.
    From top to bottom: \textsc{Chamfer-AE} (red), \textsc{EMD-AE} (green) , \textsc{SSW-AE} (magenta), \textsc{ASW-AE} (gray) and \textsc{MSW-AE} (navy) and \textsc{GSW-AE} (aqua).
  }
  \label{fig:pointgen}
\end{figure*}

As in reconstruction, we measure the effect of varying the number of slices when computing sliced Wasserstein distance for the registration task. 
The result is shown in Table~\ref{tab:registration_supp-full}.
In the registration task, increasing the number of slices helps improve the performance by more than 2\% on average (Table~\ref{tab:registration_supp-full}).

\section{Point cloud generation}
\label{sec:generation}
In Figure~\ref{fig:pointgen} and Table~\ref{tab:pointgen_supp}, we provide qualitative and quantitative results as mentioned in section Point cloud generation on page 6 of the main paper. \textcolor{black}{As we can see, \textsc{MSW} archives best performance among \textsc{SW} variants in generation tasks.} 

\end{document}